%% file: main.tex
\documentclass[sigconf,authorversion,nonacm]{acmart}
\AtBeginDocument{%
  }
\setcopyright{none}
\renewcommand\footnotetextcopyrightpermission[1]{}

\usepackage{bm,bbm,bbold, enumitem}

\usepackage{multirow}
\usepackage{subcaption}
\usepackage{graphicx}
\usepackage{algorithm}
\usepackage{algpseudocode}
\usepackage{hyperref}
\usepackage{booktabs}
\usepackage{longtable}
\usepackage{diagbox}
\usepackage{mathtools}

\hypersetup{
	colorlinks = true,
	urlcolor = {blue},
	linkcolor = {niceRed},
	citecolor = {blue}
}
\usepackage[nameinlink]{cleveref}

\newtheorem{problem}{Problem}
\newtheorem{proposition}{Proposition}
\newtheorem{assumption}{Assumption}
\crefname{assumption}{Assumption}{Assumptions}

\def\vec{\bm}

\def\A{\mathbb A}
\def\B{\mathbb B}
\def\D{\mathcal J}
\newcommand{\uncov}[1]{\mathcal R(C_k, #1)}

\begin{document}
\title{GALACTIC: Global and Local Agnostic Counterfactuals for Time-series Clustering}

\author{Christos Fragkathoulas}
\email{ch.fragkathoulas@athenarc.gr}
\affiliation{
  \institution{University of Ioannina}
   \city{Ioannina}
  \country{Greece}
}
\affiliation{
  \institution{Archimedes, Athena Research Center}
  \city{Athens}
  \country{Greece}
}

\author{Eleni Psaroudaki}
\email{h.psaroudaki@athenarc.gr}
\affiliation{
  \institution{Archimedes, Athena Research Center}
  \city{Athens}
  \country{Greece}
}

\author{Themis Palpanas}
\email{themis@mi.parisdescartes.fr}
\affiliation{
\institution{Universit\'e Paris Cit\'e, LIPADE, F-75006}
\city{Paris}
  \country{France}
}
\affiliation{
  \institution{Archimedes, Athena Research Center}
  \city{Athens}
  \country{Greece}
}

\author{Evaggelia Pitoura}
\email{pitoura@uoi.gr}
\affiliation{
  \institution{University of Ioannina}
  \city{Ioannina}
  \country{Greece}
}
\affiliation{
  \institution{Archimedes, Athena Research Center}
  \city{Athens}
  \country{Greece}
}

\renewcommand{\shortauthors}{Fragkathoulas Christos et al.}

\begin{abstract}
Time-series clustering is a fundamental tool for pattern discovery, yet existing explainability methods, primarily based on feature attribution or metadata, fail to identify the transitions that move an instance across cluster boundaries. While Counterfactual Explanations (CEs) identify the minimal temporal perturbations required to alter the prediction of a model, they have been mostly confined to supervised settings. This paper introduces GALACTIC, the first unified framework to bridge local and global counterfactual explainability for unsupervised time-series clustering.
At instance level (local), GALACTIC generates perturbations via a cluster-aware optimization objective that respects the target and underlying cluster assignments. At cluster level (global), to mitigate cognitive load and enhance interpretability, we formulate a representative CE selection problem. We propose a Minimum Description Length (MDL) objective to extract a non-redundant summary of global explanations that characterize the transitions between clusters. We prove that our MDL objective is supermodular, which allows the corresponding MDL reduction to be framed as a monotone submodular set function. This enables an efficient greedy selection algorithm with provable $(1-1/e)$ approximation guarantees.
Extensive experimental evaluation on the UCR Archive demonstrates that GALACTIC produces significantly sparser local CEs and more concise global summaries than state-of-the-art baselines adapted for our problem, offering the first unified approach for interpreting clustered time-series through counterfactuals.
\end{abstract}

\keywords{Counterfactual Explanations, Clustering, Explainability, Time-Series}

\maketitle

\input{sections/intro}
\input{sections/related}
\input{sections/problem}
\input{sections/approach}
\input{sections/experiments}
\input{sections/conclusions}

\section{Acknowledgments}
% This work has been partially supported by project MIS 5154714 of the National Recovery and Resilience Plan Greece 2.0 funded by the European Union under the NextGenerationEU Program.

This work has been partially supported by project MIS 5154714 of the National Recovery and Resilience Plan Greece 2.0 funded by the European Union under the NextGenerationEU Program, Horizon projects TwinODIS ($101160009$) and DataGEMS ($101188416$), $Y \Pi AI \Theta A$ \& NextGenerationEU project HARSH ($Y\Pi3TA-0560901$) that is carried out within the framework of the National Recovery and Resilience Plan “Greece 2.0” with funding from the European Union - NextGenerationEU.

% This work has been partially supported by Horizon projects TwinODIS ($101160009$) and DataGEMS ($101188416$),
% $Y \Pi AI \Theta A$ \& NextGenerationEU project HARSH ($Y\Pi
% 3TA-0560901$) that is carried out within the framework of the National
% Recovery and Resilience Plan “Greece 2.0” with funding from the European
% Union – NextGenerationEU, and by project MIS 5154714 of the National Recovery and Resilience Plan Greece 2.0 funded by the European Union under the NextGenerationEU Program.
\bibliographystyle{ACM-Reference-Format}
\bibliography{main}
\clearpage
\appendix
\input{sections/appendix}
\end{document}

%% file: sections/intro.tex
\section{Introduction}
Clustering is a core tool for analyzing time-series data across domains where large collections of temporal signals must be grouped to reveal shared structure without supervision.
Applications span from finance~\cite{fu2001pattern}, neuroscience~\cite{goutte1999clustering}, robotics~\cite{oates1999clustering}, biology~\cite{tran2002fuzzy}, speech processing~\cite{niennattrakul2007clustering}, multimedia~\cite{dembele2003fuzzy}, and critical infrastructure~\cite{bariya2021k}. 
In such settings, clustering acts as a critical dimensionality reduction step, transforming massive volumes of temporal signals into a manageable set of behavioral regimes. However, the utility of these clusters is often bottlenecked by their opacity. Practitioners rarely need just a label; they require an understanding of \emph{why} a series belongs to a specific cluster and \emph{how} its behavior must evolve to transition to one. Without such actionable insights, clusters remain descriptive artifacts rather than decision-supporting tools.

Unfortunately, time-series clustering pipelines are frequently opaque in practice.
Whether due to proprietary logic, complex non-linear distance metrics (e.g., DTW), or multi-stage feature transformations, the effective assignment logic is often inaccessible~\cite{bonifati2022time2feat, schlegel2025towards}. Current explainability research broadly follows two paths: \emph{inherently interpretable models} and \emph{post-hoc surrogates}. The former is often architecture-specific, embedding explanations into the clustering model itself~\cite[e.g.,][]{bertsimas2021interpretable,boniol2025k}. The latter typically provides only static summaries or surrogate models~\cite[e.g.,][]{bernard2012guided,ozyegen2022interpretable} that fail to offer a unified path from local instance-level reasoning to compact, global cluster-level insights.

Counterfactual explanations (CEs) offer a compelling solution by identifying minimal, actionable changes to an input that induce a label flip~\cite{wachter2017counterfactual}. While in time-series classification, several CE frameworks have been proposed \cite[e.g.,][]{ates2021counterfactual,lang2023generating,bahri2022temporal,hollig2022tsevo}, existing methods often struggle with temporal coherence, frequently overwrite characteristic structure, and fail to localize changes to truly discriminative regions~\cite{lang2023generating,bahri2022temporal,hollig2022tsevo}. 
More importantly, the existing methods that are designed for supervised labels do not address the unsupervised nature of discovering transitions between latent time-series clusters or the combinatorial challenge of distilling thousands of local counterfactuals into a concise global summary.

We address these gaps with \textsc{Galactic}, \emph{a unified framework for local and global counterfactual explainability in time-series clustering, agnostic to the underlying clustering procedure}.
At the \emph{local} level, \textsc{Galactic} employs an importance-guided gradient search that restricts edits to discriminative temporal regions identified via subgroup-level permutation analysis. This projection-based optimization ensures that local explanations are not only low-cost but also structurally preserved, respecting the characteristic shape of the underlying time-series.

At the \emph{global} level, we move beyond isolated instances to identify \emph{perturbations}: adjustment patterns that explain transitions for entire populations. While traditional global recourse methods (for supervised models on tabular data) yield incomparable Pareto fronts that require manual trade-off selection due to the multiobjective nature of the problem~\cite{kavouras2025glance, ley2023globe,fragkathoulas2025facegroup}, \textsc{Galactic} reformulates global selection as an Information-Theoretic Model Selection problem. By employing the \emph{Minimum Description Length (MDL)} principle~\cite{rissanen1978modeling}, we transform the competing goals of coverage and complexity into a single, principled objective. MDL automatically identifies the ``optimal'' summary by selecting a non-redundant set of perturbations that maximizes explanatory power without arbitrary weight tuning.

Furthermore, we prove that our MDL objective behaves as a monotone submodular set function. This allows us to provide efficient greedy selection strategies with $(1 - 1/e)$ approximation guarantees, alongside a scalable hierarchical refinement mechanism that adapts to the internal density and complexity of cluster subgroups.
To summarize, our \emph{contributions} are:
\begin{enumerate}[wide]
    \item \textbf{Problem formulation.} We formalize counterfactual explainability for time-series clustering at both \emph{instance} (local) and \emph{cluster/subgroup} (global) level in a model-agnostic setting.
    \item \textbf{Local counterfactuals with structural constraints.}
    We introduce a constrained gradient optimization that utilizes subgroup and cluster specific temporal importance to produce sparse, shape-preserving counterfactuals.
    \item \textbf{Global counterfactual summaries via MDL.}
    We propose a parameter-free model selection objective based on the MDL principle for identifying the set of global perturbations, resolving the traditional multi-objective trade-off. We characterize the submodularity of the MDL objective, deriving a greedy selection strategy with guaranteed approximation bounds.
    \item \textbf{Adaptive Hierarchical Refinement.} We develop an incremental refinement mechanism that utilizes the underlying density of cluster subgroups to provide tunable global explanations. This allows for a variable number of counterfactuals that adapt to the internal structural complexity of each partition and also allows for high-dimensional scalability.
    \item \textbf{Thorough Experimental Evaluation.} We demonstrate that \textsc{Galactic} produces sparser local CFs and more concise, higher-coverage global summaries than adapted state-of-the-art baselines.
\end{enumerate}

The paper structure is: \Cref{sec:related} summarizes the related work; \Cref{sec:problem} introduces the local and global problem for counterfactual explainability in time-series clustering; \Cref{sec:galactic} presents \textsc{Galactic} framework; \Cref{sec:experiments} presents the experimental evaluation; and finally, \Cref{sec:conclusions} concludes with future work.

%% file: sections/related.tex
\section{Related Work}
\label{sec:related}
The pursuit of explainability in time-series clustering has followed two distinct paths: intrinsic model design and post-hoc explanations. Intrinsic methods aim to create ``self-explanatory'' clusters by integrating interpretability directly into the architecture. For instance, Bertsimas et al. \cite{bertsimas2021interpretable} utilize hierarchical clustering paired with decision trees to optimize variable importance, while T-DPSOM \cite{manduchi2021t} leverages variational autoencoders and self-organizing maps to maintain a topology-preserving latent space. Other approaches, such as k-Graph \cite{boniol2025k} and Time2Feat \cite{bonifati2022time2feat}, identify ``graphoids'' or a rich set of spectral and temporal features to characterize clusters.

Despite these advances, most applications rely on post-hoc explanation methods, which interpret black-box cluster assignments after they are formed. Early efforts in this space focused on metadata ranking \cite{bernard2012guided, sarkar2016visual} or training surrogate classifiers to apply local interpretability tools like SHAP or Grad-CAM \cite{ozyegen2022interpretable}. However, as noted by Ntekouli et al. \cite{ntekouli2024explaining}, these methods often suffer from a ``localization gap''. While they can identify which variables are globally important, they frequently fail to provide evidence of specific temporal segments or timesteps that drive cluster membership, often collapsing into broader feature-level summaries rather than precise temporal insights. 

Counterfactual Explanations (CEs) offer a more prescriptive alternative by identifying minimal changes needed to alter a prediction. While the foundational work by Wachter et al. \cite{wachter2017counterfactual} introduced gradient-based optimization for this purpose, in its time-series adaptation, it lacked temporal awareness, leading to ``noisy'' or non-contiguous edits. 
The field has since evolved toward localized and sparse edits. TSEvo \cite{hollig2022tsevo} employs a multi-objective evolutionary algorithm to optimize for proximity, sparsity, and plausibility using segment-based crossovers. 
Glacier \cite{wang2024glacier} introduces locally constrained counterfactuals where edits are guided by constraint vectors to favor specific intervals. 
Similarly, CEI \cite{yamaguchi2024learning} constrains edits to a fixed number of contiguous temporal intervals, using a combinatorial identity to enforce interval sparsity. 
While these methods, alongside others like Sub-SpaCE~\cite{refoyo2024sub} and Native Guide~\cite{delaney2021instance}, achieve high local fidelity in supervised classification, they are not designed for the unsupervised nature of clustering. Current counterfactual clustering explanations are often tied to specific algorithmic architectures \cite[e.g.,][]{vardakas2025counterfactual} and lack the specialized operators required for the temporal dependencies of time-series data \cite{spagnol2024counterfactual}.

\textsc{Galactic} bridges the gap between time-series counterfactual search and the unsupervised setting of clustering by operating on cluster assignments queried through a surrogate. Furthermore, while these existing frameworks are focused on the instance (local) level, \textsc{Galactic} also introduces a cluster (global) level approach. By framing the selection representative temporal ``perturbations'' as a Minimum Description Length (MDL) optimization problem, it provides a unified approach that reveals the temporal structure of clusters across multiple granularities (instance, subgroup, cluster).
At the global level, we generate a set of local CFEs as candidate perturbations cluster/subgroup-level temporal adjustment patterns, and formulate their selection as a Minimum Description Length optimization problem.
Overall, this yields concise explanations at multiple granularities (cluster, subgroup, and instance), revealing which temporal regions matter and what minimal adjustments shift meaningful portions of a cluster toward others.

%% file: sections/problem.tex
\section{Problem Definition}
\label{sec:problem}
\subsection{Preliminaries}
Let $\mathcal{T} = \mathbb{R}^T$ denote the space of univariate time series of fixed length $T$. An instance $\vec x \in \mathcal{T}$ is represented as a sequence of temporal observations $\vec x = [x_1, \dots, x_T]$. We model $x$ as a concatenation of $M$ contiguous, non-overlapping segments $\mathcal{S}(x)=\{I_1,\ldots,I_M\}$ defined by breakpoints $0=\tau_0<\tau_1<\cdots<\tau_M=T$ such that $I_m=\{\tau_{m-1}+1,\ldots,\tau_m\}$.

Given a dataset $X = \{\vec x^{(i)}\}_{i=1}^N \subset \mathcal{T}$, we consider a clustering algorithm $f: \mathcal{T} \to \{1, \dots, K\}$ that maps each instance to one of $K$ discrete clusters. The algorithm induces a partition $\mathcal{C} = \{C_1, \dots, C_K\}$ of the dataset $X$, where each cluster $C_k$ is defined as the set of instances assigned to index $k$: 
\[C_k = \{x^{(i)} \in X \mid f(x^{(i)}) = k\}\]
By definition, these clusters are disjoint ($C_i \cap C_j = \emptyset$ for $i \neq j$) and exhaustive ($\bigcup_{k=1}^K C_k = X$). 

\subsection{Probabilistic Surrogate of Cluster Assignments}
While some clustering algorithms provide an explicit decision rule, many state-of-the-art methods, particularly spectral and graph-based approaches, are \textit{transductive}, meaning they lack an out-of-sample mapping function $f$ without requiring model retraining (or learning an explicit extension) \cite{bengio2003oos,alzate2010multiway,luxburg2007tutorial}. Even for inductive methods, the underlying logic may be a ``black-box'' providing only a discrete partition $\mathcal{C}$ on $X$. To maintain a model-agnostic framework, we frame the explanation task as the problem of explaining the observed assignments by learning a probabilistic surrogate classifier $g: \mathcal{T} \to [0,1]^{K}$. Here, $g_k(\vec x)$ denotes the predicted probability that instance $\vec x$ belongs to cluster $C_k$. We define the predicted assignment as $\arg\max_{k \in [K]} g_k(\vec x)$, which we denote simply as $\arg\max g(\vec x)$. 

We frame this as a supervised learning problem on the induced partition $\mathcal{C}$. Specifically, we optimize the parameters of $g$ to maximize the log-likelihood of the observed assignments $f(\vec x^{(i)})$ over $X$ via a categorical cross-entropy objective:

$$\max_{g} \frac{1}{N} \sum_{i=1}^N \mathbf{\sum_{k=1}^K \mathbbm{1}\{f(\vec x^{(i)}) = k\} \log (g_k(\vec x^{(i)})}),$$
where $g(\vec x^{(i)}) \in \Delta^{K-1}$ is the predicted probability distribution over clusters, subject to the surrogate achieving high empirical accuracy:
$$\frac{1}{N} \sum_{i=1}^N \mathbbm{1} \{\arg\max_{k\in [K]} g(\vec x^{(i)}) = f(\vec x^{(i)})\} \approx 1.$$

Crucially, because $f$ is often defined only over the specific distribution of $X$, the surrogate $g$ is required only to capture the local decision boundaries within the existing data manifold. We make no claim of generalization to out-of-distribution samples; instead, the surrogate serves as a differentiable or queryable proxy that enables the optimization of counterfactual explanations.

\subsection{Local Counterfactual Explanations}
For an instance $x \in C_k$, a \textbf{counterfactual explanation} $\vec x^{cf}$ is a perturbed instance $\vec x^{cf} = \vec x + \vec \delta$, where $\vec \delta \in \mathbb{R}^T$ represents a \textbf{minimal perturbation} (see e.g., \Cref{fig:perturbation_example}) such that the cluster assigment of $\vec x^{cf}$ differs from the original, i.e., $\arg\max g(\vec x^{cf}) = k'$ where $k' \neq k$. We refer to $k'$ as the target cluster.
In our framework, the target cluster $k'$ is identified as the most probable alternative assignment according to the posterior distribution of the surrogate: $k' = \arg\max_{c \neq k} g(x)$.
Formally, $x^{cf}$ is the solution to an optimization problem that balances cluster validity and perturbation cost. The \textbf{perturbation cost} $\mathrm{cost}(\vec x,\vec x^{cf})$ can be defined as any distance metric, such as $\ell_p$.  By optimizing this objective, we identify the minimal temporal adjustment needed to cross the decision boundary of the surrogate model. 

\begin{figure}
    \centering
    \includegraphics[width=0.72\linewidth]{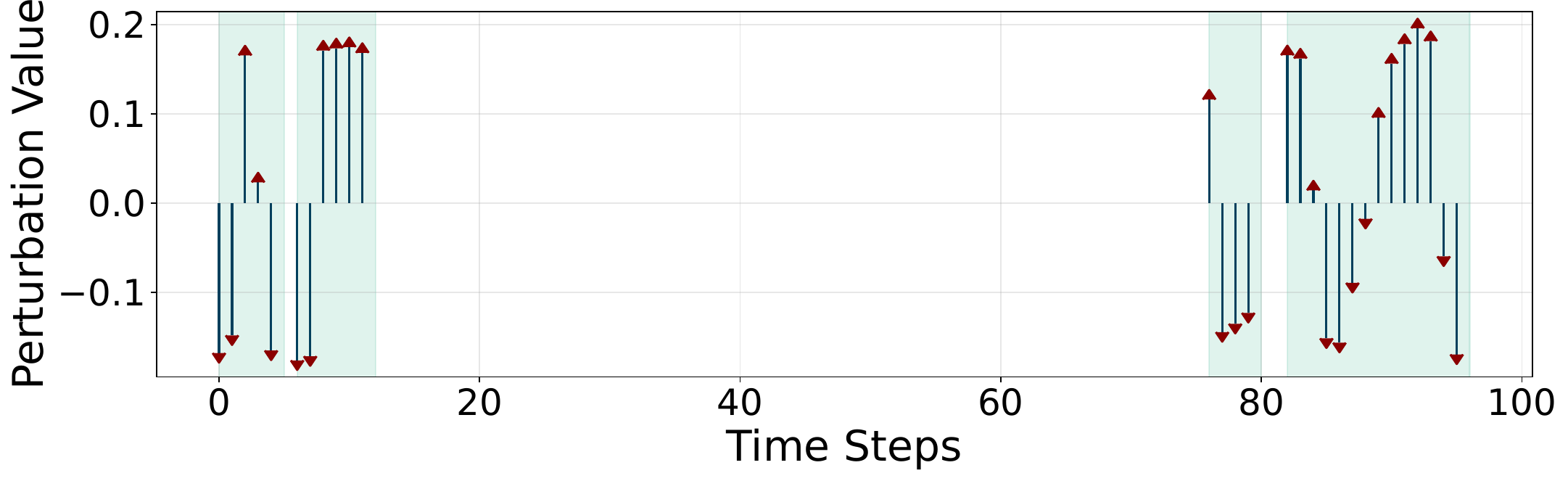}
    \caption{Example of a Perturbation $\vec \delta$}
    \label{fig:perturbation_example}
\end{figure}

\begin{problem}[Local Level Countrerfactuals]
\label{prob:local}
Given an instance $\vec x \in X$ and its surrogate assigment distribution $g(\vec x)$ with $\arg\max \ g(x) = k$, we seek the optimal counterfactual $\vec x^{cf}$ by minimizing the cost subject to a cluster-flip constraint: 
\begin{align} 
x^{cf} = \arg\min_{x'} \text{cost}(x, x') \quad \ \text{s.t.} \quad & \arg\max_{c\in [K]} g(x') = k', \quad k' \neq k  \notag
\end{align}
\end{problem}

\paragraph{Desired properties for time-series clustering counterfactual explanations} A high-quality local counterfactual for time series must satisfy \emph{two fundamental criteria} to ensure both interpretability and utility. First, it must maintain \emph{\textbf{proximity}} \cite{wachter2017counterfactual}, remaining as close as possible to the original instance under a chosen distance function to ensure the perturbation is minimal. Second, it must exhibit \emph{\textbf{sparsity}} \cite{wachter2017counterfactual} by altering only the few critical timesteps necessary to induce a change in the assignment of the surrogate.

Finally, in the \emph{specific context of clustering}, the explanation must additionally demonstrate \emph{\textbf{structural relevance}}~\cite{boniol2025k}.
Unlike supervised counterfactuals that merely seek any valid flip, clustering explanations must align with the underlying temporal patterns that define cluster membership. This requires focusing perturbations on the most discriminative segments of the time series. These high-importance regions are the characteristic temporal patterns that determine why an instance belongs to its cluster. By prioritizing changes in these segments, the counterfactual preserves the global temporal structure while highlighting the specific local features responsible for the cluster assignment.

\subsection{Global Counterfactual Explanations}
While local explanations explain why a specific time-series was assigned to a cluster, global counterfactual explanations provide a concise summary of how the instances of an entire cluster $C_k$ can cross the cluster boundaries. 
The global explanation problem is often framed as a multi-objective optimization task \cite{kavouras2025glance} and can be defined as \emph{the selection of a subset $\mathbb{S} \subseteq \mathbb{\Delta}_k$ of perturbations that provide a valid recourse (cluster flip) for the largest possible number of time-series (effectiveness) while maintaining minimal cumulative cost (average flipping cost)}. $\mathbb{\Delta}_k$ represents the possibly infinite set of all valid perturbations that can plausibly serve as good global counterfactual explanations for cluster $C_k$.

The \textbf{effectiveness} (or \textbf{coverage}) of a perturbation set $\mathbb S$ represents the fraction of time-series in $C_k$ for which at least one perturbation $\vec \delta \in \mathbb S$ results in a cluster flip:    
$$\mathrm{eff}(C_k, \mathbb S) = \frac{|\mathrm{cov}(C_k, \mathbb S)|}{|C_k|} , $$ 
where the \textbf{covered population} $\mathrm{cov}(C_k, \mathbb{S})$ consists of all $\vec{x} \in C_k$ such that $\exists \vec{\delta} \in \mathbb{S}$ where $\arg\max g(\vec{x}+\vec{\delta}) \neq \arg\max g(\vec{x})$. To quantify the effort required for this transition, we define the \textbf{flipping cost} as the minimum perturbation cost among all successful candidates in $\mathbb{S}$: 
$$\mathrm{cost}_{flip}(\vec{x}, \mathbb{S})= \min \{ \mathrm{cost}(\vec{x}, \vec{x}+\vec{\delta}) \mid \vec{\delta} \in \mathbb{S} \text{ s.t. a flip occurs} \}.$$
The \textbf{average flipping cost} serves as a global proxy for recourse effort, calculated as the mean cost over the covered population:
$$\mathrm{afc}(C_k, \mathbb{S}) = \frac{\sum_{\vec{x} \in \mathrm{cov}(C_k, \mathbb{S})} \mathrm{cost}_{flip}(\vec{x}, \mathbb{S})}{|\mathrm{cov}(C_k, \mathbb{S})|} $$

Standard multi-objective optimization of these criteria typically yields a Pareto front of non-dominated and incomparable solutions. For instance, increasing coverage typically necessitates a higher cost or a larger set of explanations \cite{fragkathoulas2025facegroup,kavouras2025glance}. 
To resolve these competing trade-offs without resorting to arbitrary weight-tuning, we employ the \emph{Minimum Description Length (MDL) principle} \cite{rissanen1978modeling}. MDL treats the selection problem as a communication task, preferring the subset of perturbations that yields the shortest joint description of the model $M$ and the observed data $D$. Formally, we seek:
\[
M^\star \;=\; \arg\min_M L(D,M),
\qquad
L(D,M) \;=\; L(M) + L(D \mid M),
\]
where $L(M)$ represents the complexity of the global explanation set, and $L(D \mid M)$ quantifies the ``unexplained'' data instances for which no selected perturbation provides a valid cluster flip.

In the MDL framework, each subset of perturbations $\mathbb{S} \subseteq \mathbb{\Delta}_k$ acts as a hypothesis to explain the transitions of instances in $C_k$. We define the \textbf{Model Code Length} $L(M)$ to penalize complexity across three dimensions: cardinality, structural sparsity, and magnitude (proximity), ensuring that the global explanation set is both effective and concise. The description length is formulated as:
\begin{align*}
L(M) &= L(\mathbb S) \\
&= \sum_{\vec \delta \in \mathbb S} \bigl( \mathrm{bits}(\Vert\vec \delta\Vert_0) + \log_2 (\Vert \vec \delta\Vert_1) \bigr) + (|\mathrm{cov}(C_k, \mathbb S)| + |\mathbb S|) \cdot p\_sz
\end{align*}
The first term, $\text{bits}(\|\vec{\delta}\|_0)$, encodes the indices of the perturbed timesteps, where a universal code for integers penalizes non-sparse edits by requiring longer bit-lengths for complex structures. The second term, $\log_2 (\|\vec{\delta}\|_1)$, captures the perturbation magnitude; following information-theoretic precision arguments, the bits required to represent a real-valued magnitude scale logarithmically, thereby favoring proximal edits. Finally, the assignment mapping term $(|\mathrm{cov}(C_k, \mathbb{S})| + |\mathbb{S}|) \cdot p\_sz$ accounts for the pointers required to associate each covered instance in $C_k$ with its ``best-fit'' perturbation in $\mathbb{S}$. This structured approach transforms the multi-objective Pareto trade-off into a single communication task, where the most informative model is the one that minimizes the total description length of the cluster transitions.

The \textbf{Data Length} $L(D \mid M)$ quantifies the ``failure to explain'' by penalizing instances that remain uncovered by the selected set $\mathbb{S}$. In an information-theoretic sense, instances that do not change clusters under any $\vec{\delta} \in \mathbb{S}$ must be encoded as outliers. To avoid the prohibitive computational cost of calculating individual counterfactuals for every outlier, we approximate their encoding cost using a ``worst-case'' reference perturbation $\vec{\delta}^\ast$, defined as the most complex candidate in $\mathbb{\Delta}_k$. The data code length is formulated as:
\begin{align*}
L(D &\mid M) = L(C_k \mid \mathbb S) = \notag \\ &=  \big|C_k \setminus \mathrm{cov}({C_k, \mathbb S})\big|  \cdot \Bigl(
\mathrm{bits}\bigl(\Vert\vec \delta^\ast\Vert_0\bigr)
+ 
\log_2 (\Vert\vec \delta^\ast\Vert_1)
 +2\cdot p\_sz\Bigr)
\end{align*}
This term acts as a dynamic penalty for low coverage: as the number of uncovered instances increases, the description length $L(D \mid M)$ grows linearly, forcing the MDL objective to favor the inclusion of more effective perturbations. By combining $L(M)$ and $L(D \mid M)$, the framework identifies the optimal $\mathbb S$, where the gain in coverage no longer justifies the added 
complexity, resolving the global selection problem without the need for manual hyperparameter tuning.

\begin{problem}[Global Level Counterfactuals] % mdl in problem
\label{prob:globalMDL}
Given cluster $C_k$, a maximum cardinality of the perturbation set $\mu$, find the optimal subset $\mathbb S^\ast \subseteq \mathbb \Delta_k$ that represents the solution of the following optimization problem:
\[
\mathbb S^\ast = \arg\min_{\mathbb S \subseteq \mathbb \Delta_k} \quad 
%\mathcal{J}(\mathbb S) =
L(\mathbb S) + L(C_k \mid \mathbb S) \quad\text{s.t.}\quad
|\mathbb S| \leq \mu\]
\end{problem}

%% file: sections/approach.tex
\section{GALACTIC Framework}
\label{sec:galactic}

\subsection{Local Approach}
Given a target cluster $k'$, we seek a counterfactual $\vec{x}^{cf}$ that minimizes a cost function while traversing the decision boundary of the surrogate. To ensure structural relevance, we guide the search using a multi-stage segmentation and importance-scoring framework.

\paragraph{Segmentation and Subgrouping} We first capture the distinctive temporal properties of cluster $C_k$ by computing instance-wise segmentations $\mathcal{S}(\vec{x})$ via change-point detection \cite{sivill2022limesegment}. To account for imperfect clustering and internal variability, we encode these segmentations into gap vectors $\gamma(\vec{x})$ (representing segment durations) and apply $K_{\text{seg}}$-medoids clustering. This yields a set of representative patterns $\mathcal{P}_k = \{ \mathcal{S}_1, \ldots, \mathcal{S}_R \}$ that summarizes the dominant temporal structures of the cluster. Each pattern $\mathcal{S}_r$ defines a subgroup $G_r$ within cluster $C_k$, providing a structured partition of the cluster.
\label{sec:segmentation_and_subgrouping}

\paragraph{Importance-Guided Perturbation} 
For every cluster $C_k$, we derive \textbf{a timestep importance mask} $\vec w \in \{0,1\}^T$ by evaluating the impact of segment permutations $\mathcal P$ on the accuracy of the surrogate model $g$. We define the importance of segment $I_m$ as $\operatorname{imp}(I_m) = | a_c - \frac{1}{B} \sum a_{c,m}^{(b)} |$, where $a_c$  is the baseline accuracy for the cluster and $a_{c,m}^{(b)} $ is the accuracy after shuffling the values of segment $I_m$ across the cluster population (in the $b$-th repetition.
We propagate segment scores to timesteps and binarize using a quantile threshold (typically the 75th percentile). This results in a sparse mask $\vec{w}$ where $w_t = 1$ only for timesteps residing within highly discriminative segments.
Similarly, we derive subgroup-specific importance vectors by restricting this permutation analysis to instances within the specific subgroup $G_r$ and its corresponding representative pattern $\mathcal{S}_r$.

Rather than a penalty in the loss function, this mask $\vec{w}$ is applied as a hard constraint directly to the gradient during optimization: $\nabla \leftarrow \nabla_{\vec{x}'} \mathcal{L} \odot \vec{w}$. This projection ensures that the optimizer is restricted from altering ``low-importance'' regions, effectively preserving the non-discriminative parts of the time-series and forcing the search to find a solution within the structural boundaries of the cluster. We consider three \textbf{strategies} for $w$ for each $\vec x$: (1)~Source-group importance, focuses on regions that define the current assignment (\Cref{fig:source}); (2) Target-cluster importance, identifies regions required to achieve membership in $C_{k'}$ (\Cref{fig:target}); and (3) Combined importance, which aligns breakpoints from both source group and target cluster to highlight mutually critical intervals (\Cref{fig:combined}). 
Detailed formalizations 
are provided in the Appendix \Cref{sec:strategies}.

\begin{figure}
  \centering
  \begin{subfigure}{0.48\columnwidth}
    \centering
    \includegraphics[width=0.95\linewidth]{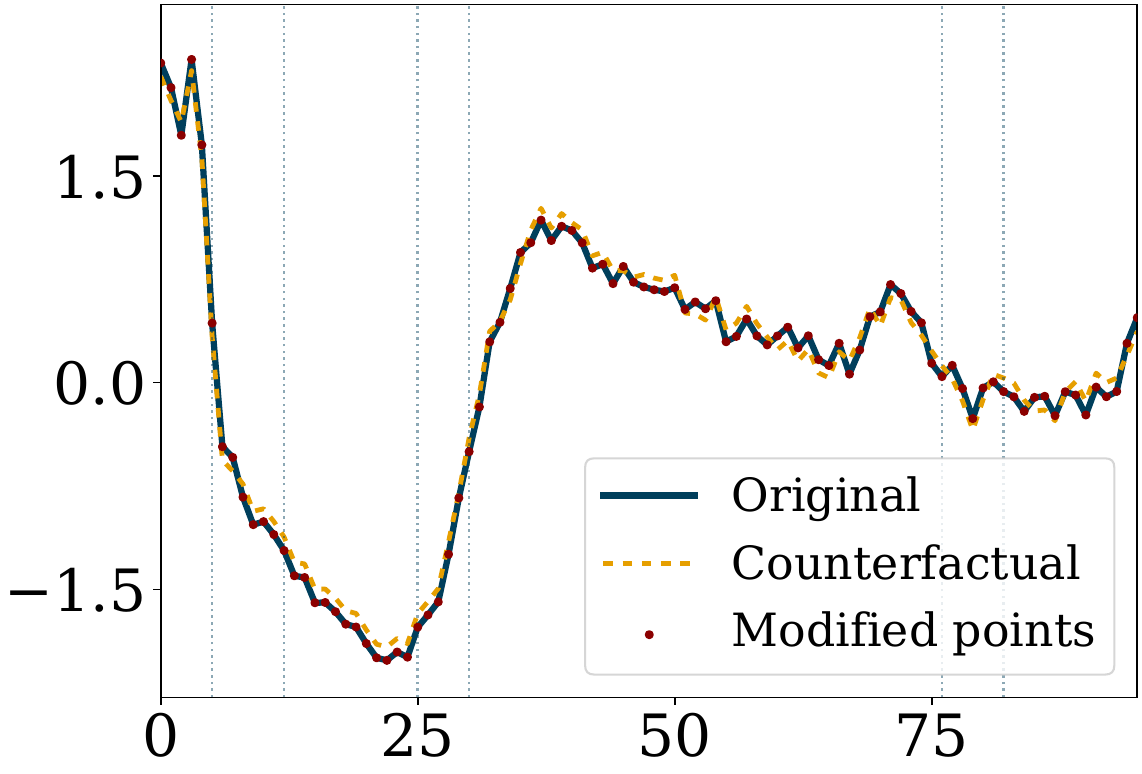}
    \caption{Baseline}
    \label{fig:baseline}
  \end{subfigure}\hfill
  \begin{subfigure}{0.48\columnwidth}
    \centering
    \includegraphics[width=0.95\linewidth]{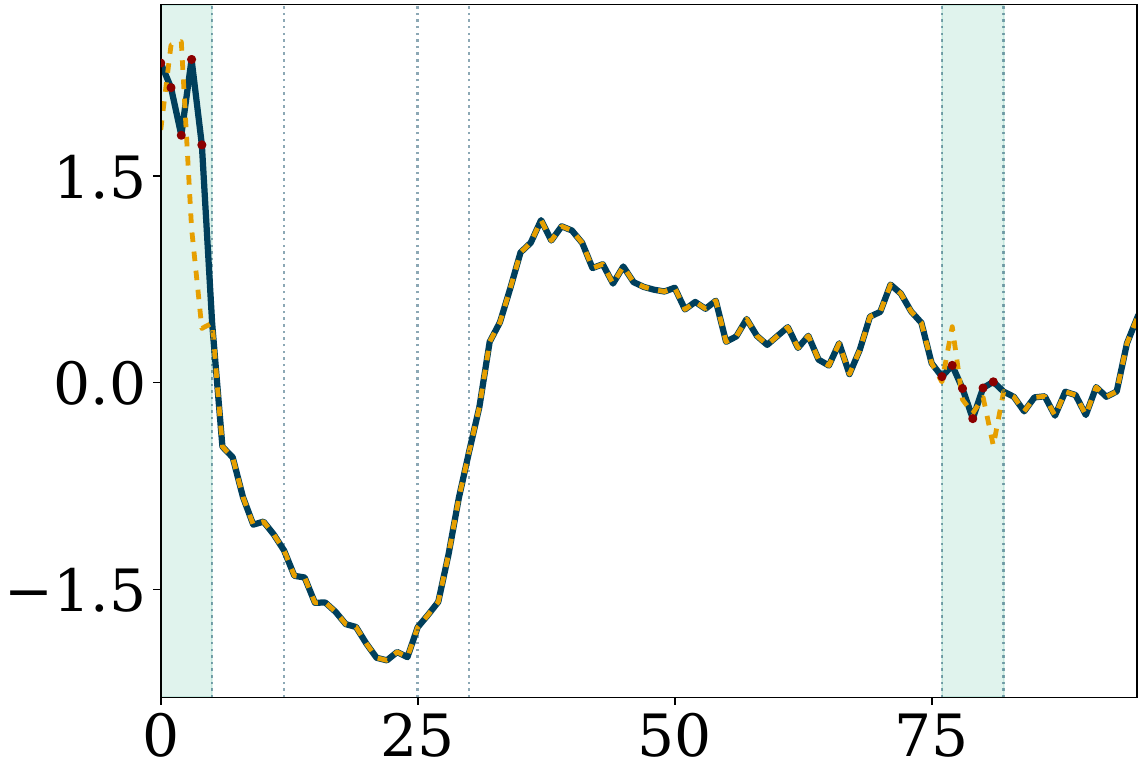}
    \caption{Source}
    \label{fig:source}
  \end{subfigure}
  \medskip
  \begin{subfigure}{0.48\columnwidth}
    \centering
    \includegraphics[width=0.95\linewidth]{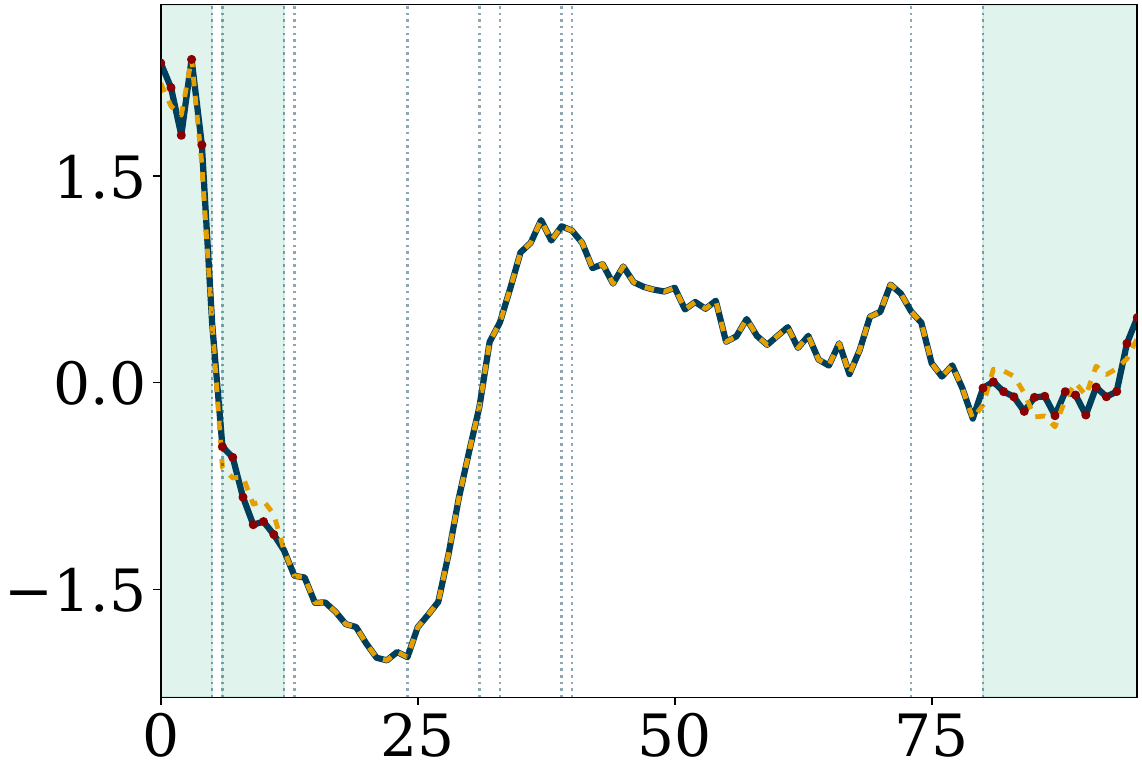}
    \caption{Target}
    \label{fig:target}
  \end{subfigure}\hfill
  \begin{subfigure}{0.48\columnwidth}
    \centering
    \includegraphics[width=0.95\linewidth]{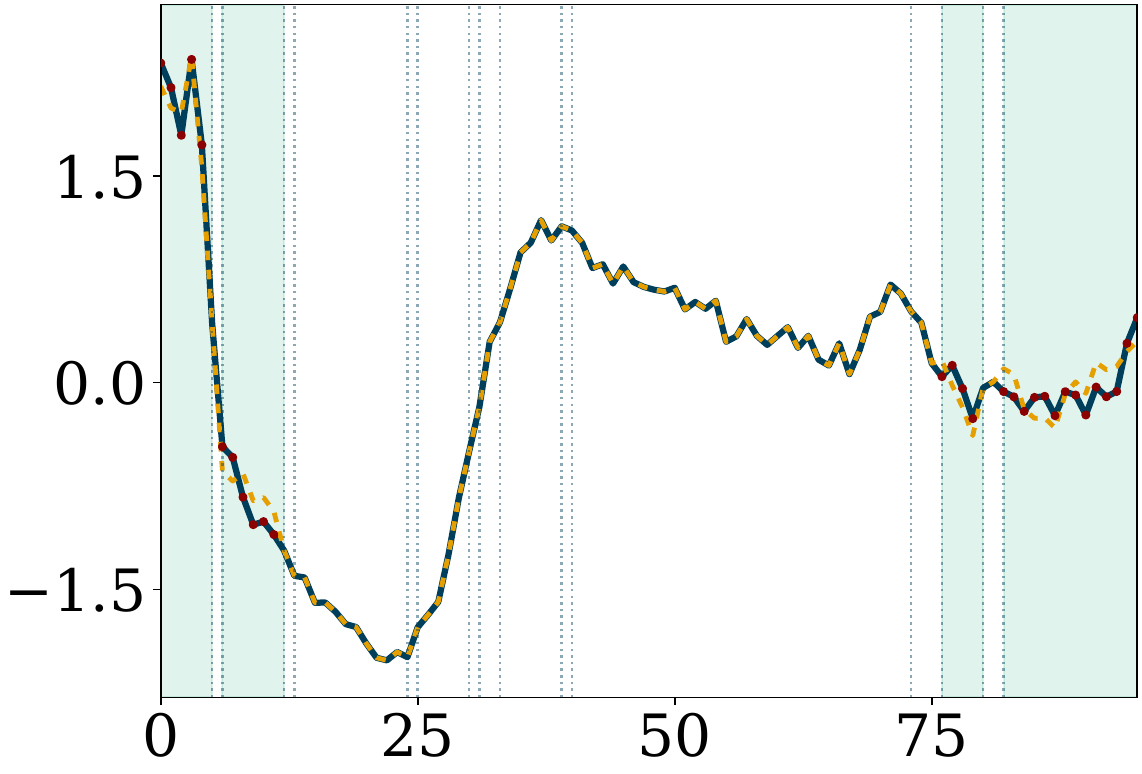}
    \caption{Combined}
    \label{fig:combined}
  \end{subfigure}
  \caption{Comparison of local explanations for a time-series generated under different importance weighting strategies.}
  \label{fig:segment_weightings}
\end{figure}

\paragraph{Gradient-Based Optimization} 
We formalize the local counterfactual search as a constrained optimization problem, as detailed in \Cref{algo:local_refined}. We initialize the search by adding small Gaussian noise $\mathcal{N}(0, \epsilon^2)$ to the original series $\vec{x}$, restricted to the important regions defined by $\vec{w}$. The objective function $\mathcal{L}$ balances a proximity term $\mathcal{L}_{prox}$, ensuring the counterfactual remains close to the original instance, and a validity loss $\mathcal{L}_{flip}$ that drives the instance across the cluster boundary. Depending on the availability of a specific target cluster $tk$, $\mathcal{L}_{flip}$ either minimizes the negative log-likelihood of the target cluster $g_t(\vec{x}')$ or maximizes the distance from the distribution of the source cluster $g_s(\vec{x}')$.

At each iteration, we compute the gradient $\nabla_{\vec{x}'} \mathcal{L}$ and apply the structural importance mask via an element-wise product $\nabla_{\vec{x}'} \mathcal{L} \odot \vec{w}$. This projection ensures that the optimization updates are strictly confined to discriminative segments, preventing perturbations in non-essential regions. We utilize the Adam optimizer to update $\vec{x}'$ until the predicted label flips and the loss converges within a tolerance $\tau$ over $p$ consecutive iterations. If multiple valid counterfactuals are found, the algorithm returns the candidate $\vec{x}^{cf}$ that minimizes the total cost, thereby satisfying the requirements for sparsity, proximity, and structural relevance.
\begin{algorithm}[t]
\caption{\textsc{Galactic-L}}
\label{algo:local_refined}
\small
\begin{algorithmic}[1]
\Require original series $\vec{x} \in \mathbb{R}^T$, surrogate $g$, 
% importance mask $\vec{w}$, 
step size $\eta$, iterations $N_{iter}$,  weighting strategy $\mathcal{W}$, tolerance $\tau$, patience $p$, jitter $\epsilon$
\Ensure counterfactual $\vec{x}^{cf}$
\Procedure{Galactic\_L}{$\vec{x}, g, \eta, N_{iter}, \mathcal{W}, \tau, p, \epsilon$}
\State $s \leftarrow \arg\max g(\vec x)$  \Comment{Source Cluster}
\State $t \leftarrow \textproc{ResolveTarget}(\vec{x}, g)$ \Comment{Target cluster or \textit{None}}
\State $\vec w \leftarrow \textproc{GetMask}(\vec{x}, s, t, \mathcal{W})$ \Comment{Derive Importance Mask}
\State $\vec{x}' \leftarrow \vec{x} + \mathcal{N}(0, \epsilon^2)\odot \vec{w}$ \Comment{Stochastic initialization}
\State $\vec{x}^{cf} \leftarrow \text{None}, \mathcal{L}_{min} \leftarrow \infty,  \mathcal{L}_{0} \leftarrow 0, \text{count} \leftarrow 0$

\For{$i=1$ to $N_{iter}$}
    \State $\mathcal{L}_{prox} \leftarrow \|\vec{x}' - \vec{x}\|_2$ \Comment{Proximity loss}
    \If{$t \neq \text{None}$} \Comment{Target-specific push}
        \State $\mathcal{L}_{flip} \leftarrow -\log(g_t(\vec{x}') + \zeta)$ \Comment{Validity loss}
    \Else \Comment{Agnostic push-away}
        \State $\mathcal{L}_{flip} \leftarrow \log(g_s(\vec{x'}) + \zeta)$ 
    \EndIf
    
    \State $\mathcal{L}_i \leftarrow \lambda_1 \mathcal{L}_{prox} + \lambda_2 \mathcal{L}_{flip}$
    
    % Evaluation of stopping criteria
    \State $\text{is\_flipped} \leftarrow (t \neq \text{None} \land \arg\max g(\vec{x}') = t) \lor (t = \text{None} \land \arg\max g(\vec{x}') \neq s)$
    
    \If{\text{is\_flipped}}
        \If{$\mathcal{L}_i < \mathcal{L}_{min}$}
            \State $\vec{x}^{cf} \leftarrow \vec{x}', \mathcal{L}_{min} \leftarrow \mathcal{L}_i$
        \EndIf
        \State $\text{count} \leftarrow \text{count} + 1 \text{ \quad if }(|\mathcal{L}_i - \mathcal{L}_{i-1}| < \tau)  \text{ else } 0$
    \EndIf
    \If{$\text{count} \geq p$} \textbf{break} \ \EndIf
    \State $\nabla \leftarrow \nabla_{\vec{x}'} \mathcal{L} \odot \vec{w}$ \Comment{Apply structural mask to gradient}
    \State $\vec{x}' \leftarrow \text{AdamStep}(\vec{x}', \nabla, \eta)$ \Comment{Update $\vec{x}'$}
\EndFor
\State \Return $\vec{x}^{cf}$
\EndProcedure
\end{algorithmic}
\end{algorithm}

\subsection{Global Approach}
The global counterfactual explanation problem requires selecting a small set of perturbations that induces cluster transitions for the largest possible fraction of the population. Since the space of all valid temporal perturbations $\mathbb{\Delta}$ is theoretically infinite and not explicitly defined, we must construct a pool of physically plausible and structurally relevant perturbations.

\paragraph{Candidate Generation}
We leverage \textsc{Galactic-L} (\Cref{algo:local_refined}) to generate a finite candidate pool $\Delta_k = \{\vec{\delta}_1, \dots, \vec{\delta}_M\}$ by computing the perturbations $\vec{\delta}_i = \vec{x}_i^{cf} - \vec{x}_i$ obtained from a representative subset of instances in cluster $C_k$. This ensures that every candidate in our pool is a valid, gradient-verified transition that respects the cluster-specific structural constraints $\vec{w}$. By using these local ``best-responses'' as our search space, the global selection problem transforms into finding a subset $\mathbb{S} \subseteq \Delta_k$ that explains the transitions of the entire cluster through the shortest joint description length.

To ensure the candidate pool $\Delta_k$ is both computationally tractable and distributionally representative, we must carefully select the instances used for generation. Simple centroids are insufficient for time-series data as they often fail to capture the multi-modal nature of temporal clusters and may result in patterns that do not exist in the raw data. Instead, we select a distributional coreset of representative instances using the MMD-Critic algorithm~\cite{kim2016examples}.

MMD-Critic utilizes the Maximum Mean Discrepancy (MMD) to quantify the distance between the cluster distribution and the subset distribution in a Reproducing Kernel Hilbert Space (RKHS)~\cite{gretton2012kernel}. This allows us to identify a set of prototypes, which represent the dominant temporal patterns of the cluster, and criticisms, which represent low-density regions and outliers that prototypes fail to capture. By generating perturbations from both prototypes and criticisms, we ensure the candidate pool $\Delta_k$ contains transitions for both the ``average'' and the ``edge cases'' of the cluster manifold, providing a robust foundation for the subsequent MDL selection.

\paragraph{Algorithmic Variants} While the MDL objective can evaluate any subset size, interpretability requirements dictate a limit of at most $\mu_k$ global perturbations per cluster. We present four variants of the selection process. A detailed comparative evaluation of the quality-runtime Pareto frontier for these is provided in the \Cref{sec:galactic_g_varied_m}.

\emph{\textbf{1. Optimal.}}
The Optimal approach (App. \Cref{alg:optimal}) performs an exhaustive search over the power set of candidates of cardinality at most $\mu$. It evaluates $\sum_{i=1}^{\mu} \binom{|\Delta_k|}{i}$ combinations to find the global minimum of the joint description length $L(C_k, \mathbb S)$. While computationally prohibitive for large sets, due to its complexity ($\mathcal O (|\Delta_k|^{\mu_k})$), it serves as the ground-truth baseline for small-scale patterns.

\emph{\textbf{2. Greedy.}}
The Greedy approach (App. \Cref{alg:greedy}) provides an efficient approximation of the optimal solution by iteratively incorporating the perturbation that yields the maximal marginal gain in compression. Under mild assumptions (see \Cref{sec:appendix_proofs}), the MDL objective behaves as a supermodular function, rendering the reduction in total description length a monotone submodular set function. This property ensures that the greedy procedure achieves the $(1 - 1/e)$ approximation guarantee for submodular maximization~\cite{hochba1997approximation} with a computational complexity of $\mathcal{O}(\mu_k |\Delta_k|)$. The search terminates when the budget $\mu_k$ is reached or an intrinsic stopping criterion is met: if the marginal bit-cost of encoding a new perturbation exceeds its coverage benefit, the algorithm stops to preserve model parsimony and avoid overfitting.

\emph{\textbf{3. Hierarchical.}} The Hierarchical approach (\Cref{alg:hierarchical}) decomposes the selection into two phases: subgroup level selection followed by cluster-level refinement. This architecture supports two distinct operational modes:
\textbf{(a) Hierarchical Optimal.} It employs exhaustive search at the subgroup level to identify the exact MDL-minimizing perturbations for each $G \subset C_k$.With $m_G$ candidates per group and a subgroup budget $\mu_r$, Phase 1 requires $\sum_{G} \sum_{s=1}^{\mu_r} \binom{m_G}{s}$ evaluations. The resulting winners form a reduced set $\widetilde{\Delta}_k$ of size $m'$, which is then refined via a second exhaustive search requiring $\sum_{s=1}^{\mu_k} \binom{m'}{s}$ evaluations. This significantly reduces the search space compared to a flat optimal search, though it sacrifices global optimality for local precision.
\textbf{(b) Hierarchical Greedy.} This approach is designed for maximum scalability and utilizes the submodular greedy strategy at both levels. At the subgroup level, it requires $\sum_{G} \mu_r  m_G$ evaluations to populate the reduced set $\widetilde{\Delta}_k$. The final cluster-level refinement adds $\mu_k  m'$ evaluations. This variant drastically reduces the number of MDL assessments, especially when $m' \ll |\Delta_k|$, making it the most suitable approach for high-dimensional time series and massive datasets.

\begin{algorithm}[t]
\caption{\textsc{Galactic-G}: Hierarchical}
\label{alg:hierarchical}
\small
\begin{algorithmic}[1]
\Require Cluster $C_k$, groups $\mathcal{G}_k$, candidate perturbations $\Delta_{C_k}$, budget $\mu_k$, strategy $\text{algo} \in \{\texttt{optimal, greedy}\}$ 
\Ensure Global perturbation set $\mathbb{S}_{C_k}$
\Procedure{Hierarchical}{$C_k, \mathcal{G}_k, \Delta_{C_k}, \mu_k, \text{algo}$}:
    \State $\widetilde{\Delta}_{C_k} \gets \emptyset$ \Comment{Pool of candidate perturbations from group selections}
    \Statex \textbf{// Phase 1: Subgroup level selection}
    \ForAll{$G_r \in \mathcal{G}_k$}
        \State $\Delta_{G_r}\gets \textproc{GetPerm}(\Delta_{C_k}, G_r, g)$ \Comment{Get applicable candidate}
        \State $\mu_r \gets \lceil \frac{|G_r|}{|C_k|} \mu_k \rceil \cdot |\mathcal{G}_k|$ \Comment{Proportional budget allocation}
        \State $\widetilde{\Delta}_{C_k} \gets \widetilde{\Delta}_{C_k} \cup \; \textproc{SelectPerm}({\Delta}_{G_r},G_r,\mu_r, \text{algo})$ \Comment{Group MDL}
    \EndFor
    \Statex \textbf{// Phase 2: Cluster level refinement}
    \State $\mathbb{S}_{C_k} \gets \textproc{SelectPerm}(\widetilde{\Delta}_{C_k},C_k,\mu_k, \text{algo})$ \Comment{Refined cluster MDL}
    \State \Return $\mathbb {S}_{C_k}$ 
\EndProcedure
\Statex
\Procedure{SelectPerm}{$\Delta, C, \mu, \text{algo}$}:
\If{algo = \texttt{optimal}} 
\State \Return \textproc{Optimal}($\Delta, C, \mu$) \Comment{Exhaustive search for small $|\Delta|$}
\ElsIf{algo = \texttt{greedy}}
\State \Return \textproc{Greedy }($\Delta, C, \mu$)  \Comment{Greedy MDL}
\EndIf
\EndProcedure   
\end{algorithmic}
\end{algorithm}

%% file: sections/experiments.tex
\section{Experimental Evaluation}
\label{sec:experiments}
To evaluate \textsc{Galactic}, we conduct experiments on diverse datasets from the UCR Time-Series Archive~\cite{dau2019ucr} addressing four research directions: (i) \textbf{Target Selection Analysis}: Validating the superiority of the \textit{second possible cluster} policy over \textit{random} selection; 
(ii) \textbf{Weighting Strategy Efficacy}: Demonstrating the  performance of \textit{Combined} segment importance; 
(iii) \textbf{Local Evaluation}: Benchmarking \textsc{Galactic-L} against state-of-the-art local counterfactual generators; 
(iv)  \textbf{Global Evaluation}: Benchmarking \textsc{Galactic-G} approaches against adapted baselines for clustered series.

\paragraph{Datasets and Models.}
We evaluate \textsc{Galactic} on 30 datasets from the UCR Time Series Archive \cite{dau2019ucr}, spanning multiple domains (e.g., Image, Motion, Sensor). Although the UCR archive is primarily used for classification benchmarks, it is also used for clustering \cite{paparrizos2024bridging,fotakis2024efficient}.
We treat the ground-truth labels of each UCR dataset as the cluster partition (i.e., each class defines a cluster $C_k$), while subgroups are obtained by clustering the segmentation gap vectors within each cluster as described in Section~\ref{sec:segmentation_and_subgrouping}.
Appendix \Cref{tab:dataset_summary} summarizes the selected datasets, including their size, length, number of clusters, domain type, and the classification accuracy achieved by the surrogate model.
All surrogate models follow an 80\% to 20\% train-test split.
Detailed parameter selection, architectural specifications, hardware configurations, and datasets are described in \Cref{sec:technical_details}.
The source code of \textsc{Galactic} is available upon request.

\paragraph{Evaluation Metrics.}  To assess the performance of the generated counterfactuals, we report five standard measurements across all experiments:
(1)~\textbf{Effectiveness (eff)}: the percentage of instances where the search successfully flips the prediction to the target cluster;
(2)~\textbf{Average Flipping Cost (afc)}: the $L_2$ distance between the original series $x$ and the counterfactual $x^{cf}$, representing the minimality of the edit;
(3)~\textbf{Average Changed Segments (acs)}: the average number of segments in the series that were modified, i.e., the \textit{localized} changes within contiguous regions rather than scattered ones across disjoint timesteps;
(4)~\textbf{Average Changed Timesteps (act)}: the average number of individual timesteps modified, reflecting fine-grained interpretability; and
(5)~\textbf{Runtime (RT)}: the total time required to generate explanations.

\subsection{Search Policy: Target Selection}
\label{sec:target_policy_main}
For \textsc{Galactic-L}, we perform a comparative evaluation of two target selection policies: {\textbf{second possible}} (selecting the cluster with the next-highest surrogate probability), and \textbf{random} (randomly selecting a target cluster).
To compare policies across datasets, we summarize their behavior using the aggregate measurement \textit{Gain/Loss} that reflects whether a policy improves or worsens a metric on average relative to another policy (a Gain indicates a higher success rate for the method reported in the y-axis, lower costs (flipping cost, changed segments/timesteps, and runtime)).
Details can be found in the \Cref{sec:target_policies}, as well as the evaluation of a third selection policy \textbf{all\_random} (no explicit target selection) in the \Cref{sec:target_selection_appendix}.
As shown in \Cref{fig:direction_scatter}, the \textit{second possible} policy consistently achieves the highest effectiveness while minimizing average flipping cost.
In contrast, \textit{random} exhibit higher average flipping cost distances, a higher number of changed segments, and lower effectiveness, validating that the \textsc{Galactic} objective effectively leverages the latent topology of the classifier to find efficient transition paths.
\begin{figure}
  \centering
  \begin{subfigure}{0.5\columnwidth}
    \includegraphics[height=0.80\linewidth]{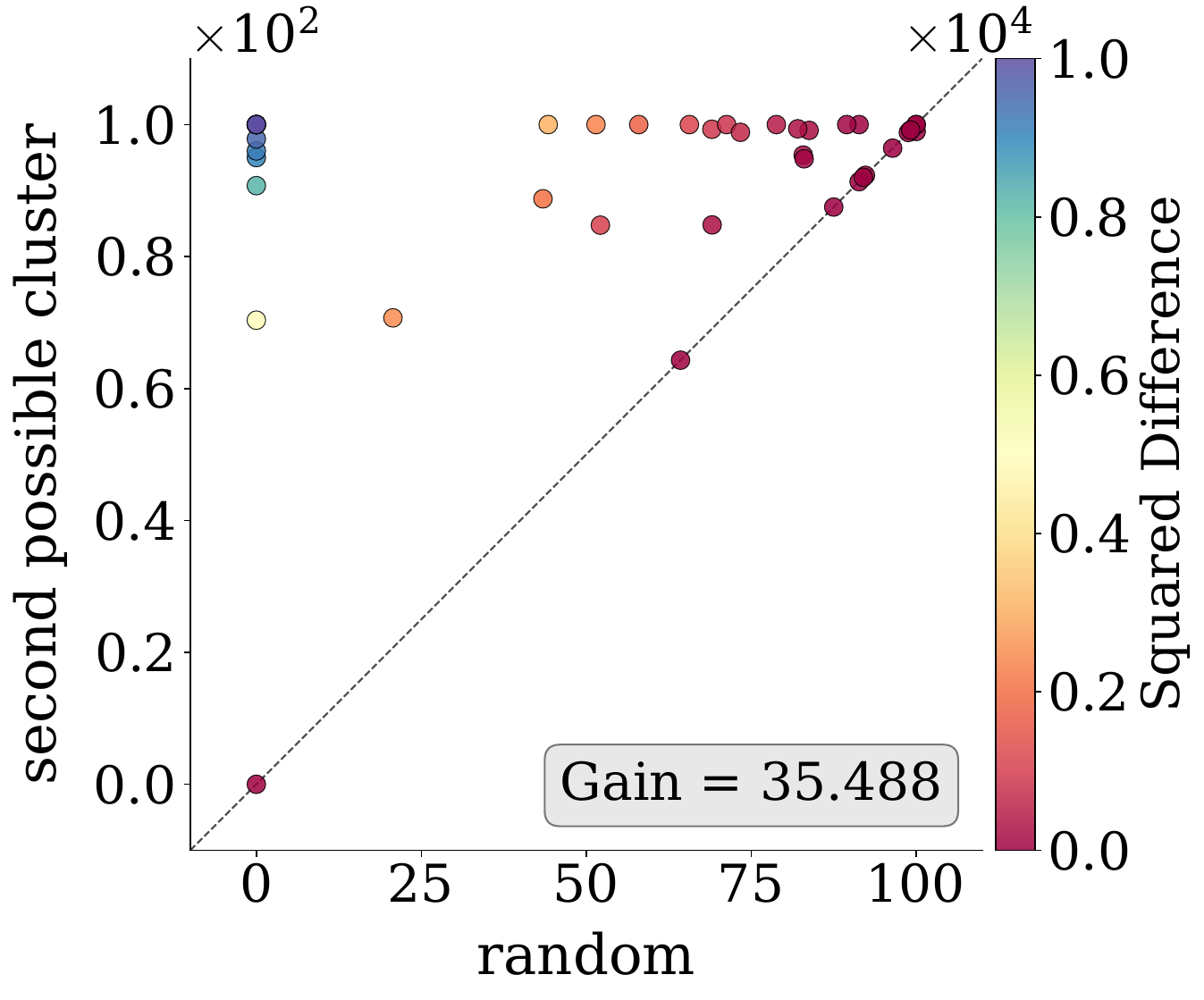}
    \caption{Effectiveness}
  \end{subfigure}\hfill
  \begin{subfigure}{0.5\columnwidth}
    \includegraphics[height=0.80\linewidth]{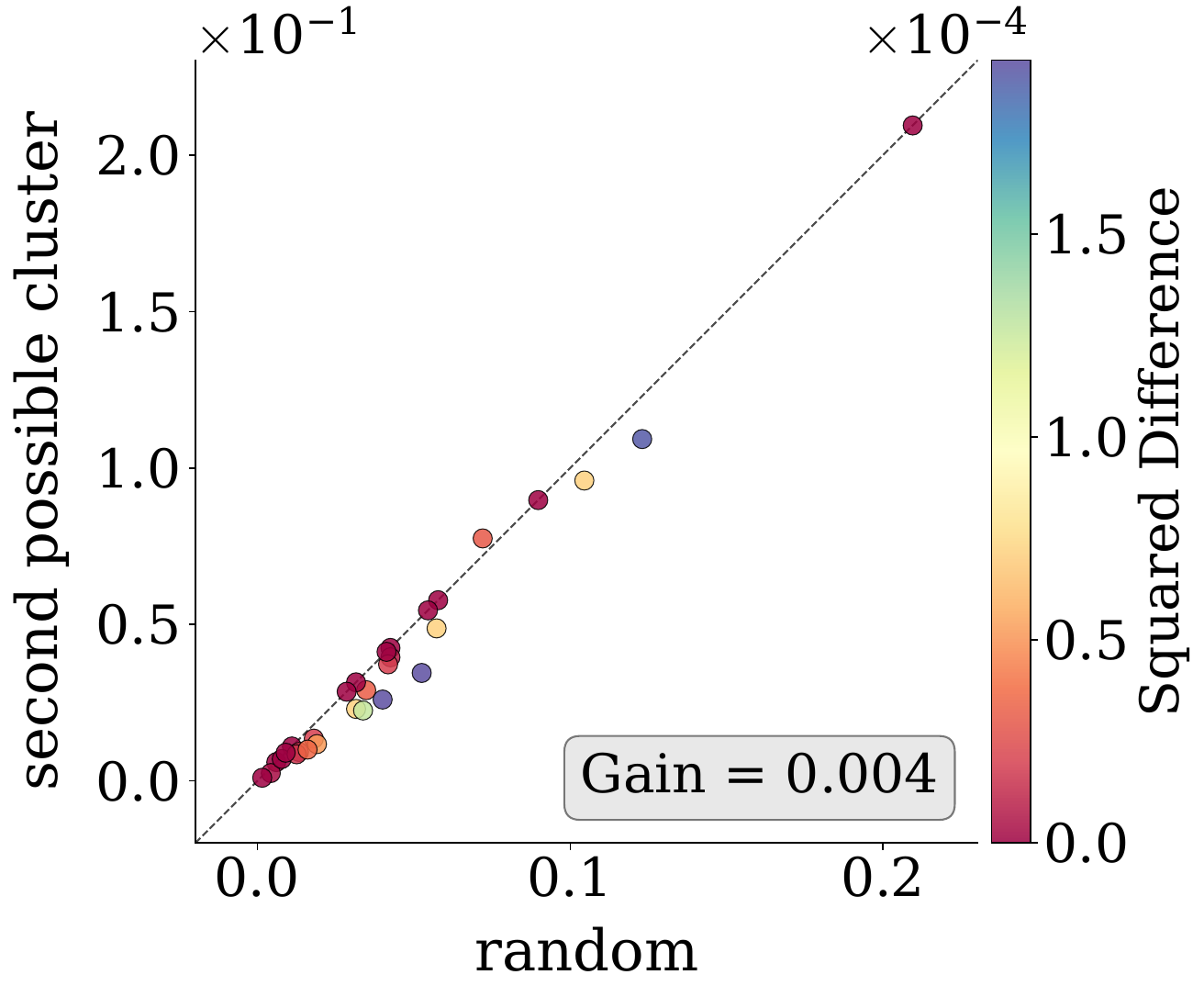}
    \caption{Avg. Flipping Cost}
  \end{subfigure}
  \medskip
  \begin{subfigure}{0.5\columnwidth}
    \includegraphics[height=0.80\linewidth]{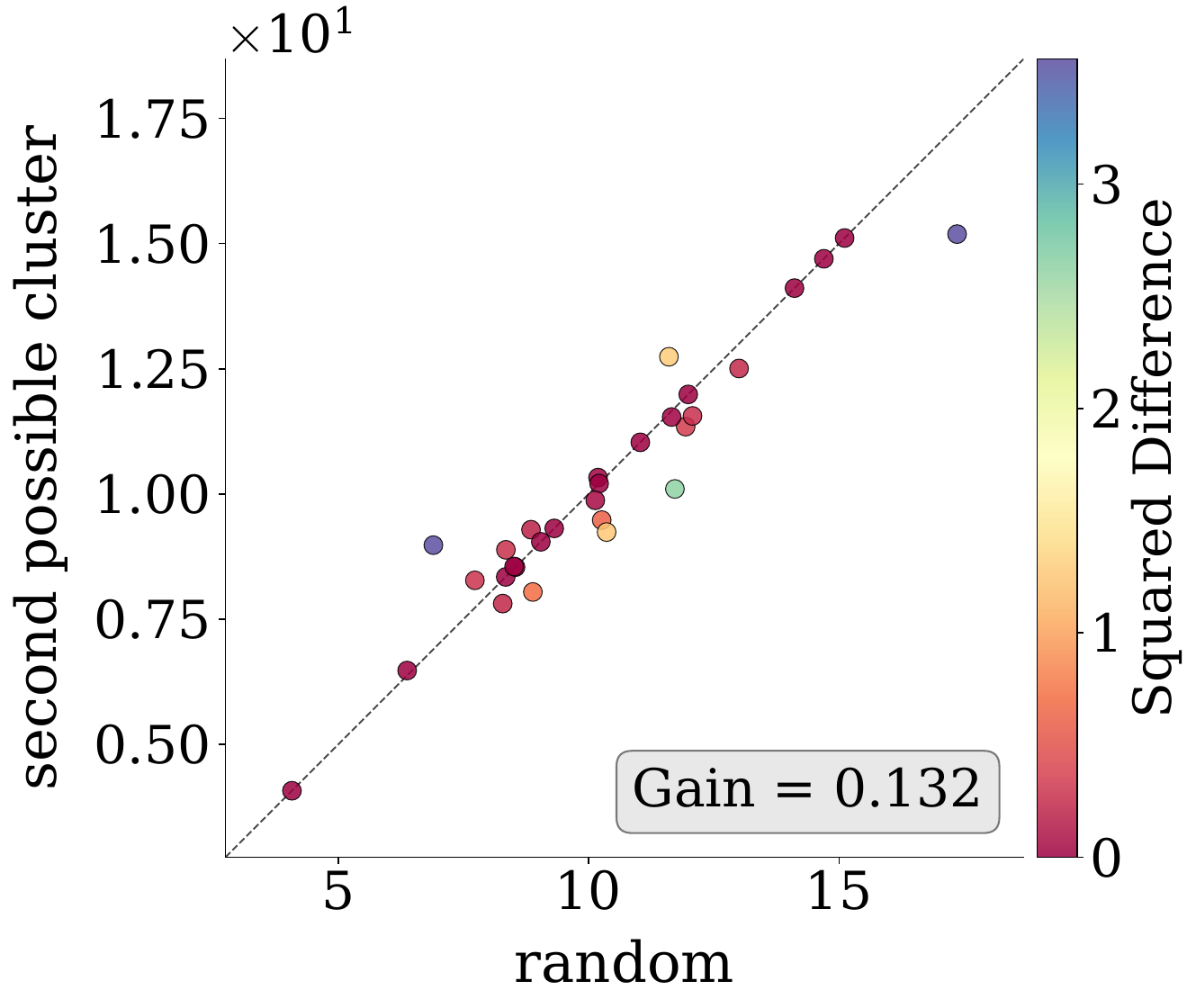}
    \caption{Avg. Changed Segments}
  \end{subfigure}\hfill
  \begin{subfigure}{0.5\columnwidth}
    \includegraphics[height=0.80\linewidth]{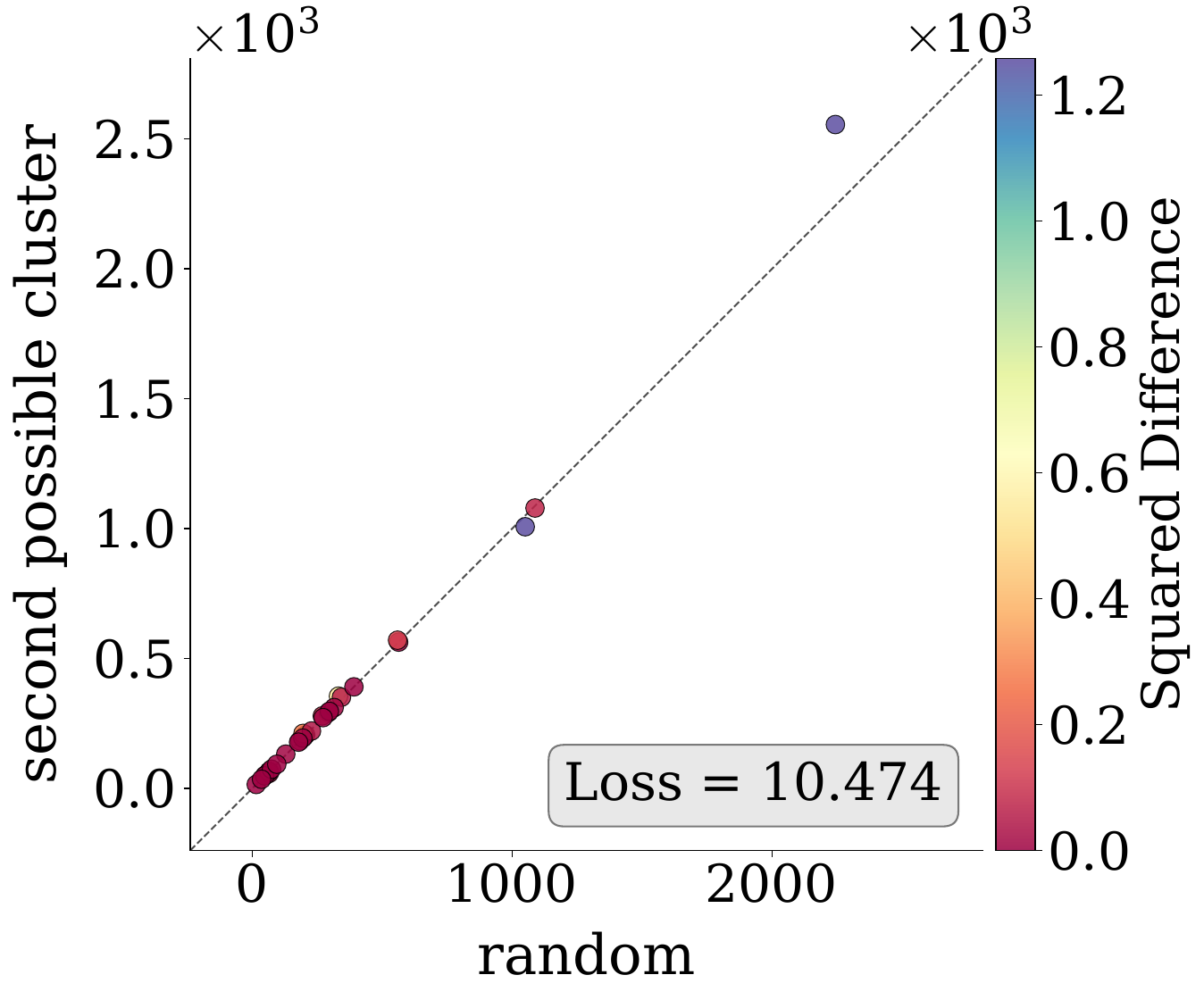}
    \caption{Avg. Changed Timesteps}
  \end{subfigure}
  \caption{\textsc{Galactic-L} Target Selection Analysis. Comparison of second possible and random policies across UCR datasets.}
  \label{fig:direction_scatter}
\end{figure}

\subsection{Importance Weighting: Sparsity and Validity}
\label{sec:importance_weighting_main}
We benchmark our three weighting strategies $\mathcal{W}$ to assess their impact on counterfactual quality against the  \textbf{Baseline} (no masking):  \textbf{Source}, \textbf{Target}, and \textbf{Combined}. 
As shown in \Cref{fig:weighting_results}, the Combined strategy yields the optimal trade-off between cluster-flip success and interpretability. While the Baseline often achieves high success rates, it suffers from poor sparsity, modifying irrelevant temporal regions. Conversely, the Combined mask restricts perturbations to segments critical to both source and target clusters, significantly reducing the number of changed segments (as evidenced by the lower average and narrower standard deviation in sparsity metrics) without compromising the cost minimality of the resulting counterfactual.

\begin{figure}
  \centering
  \begin{subfigure}{0.5\columnwidth}
    \centering
    \includegraphics[width=0.91\linewidth]{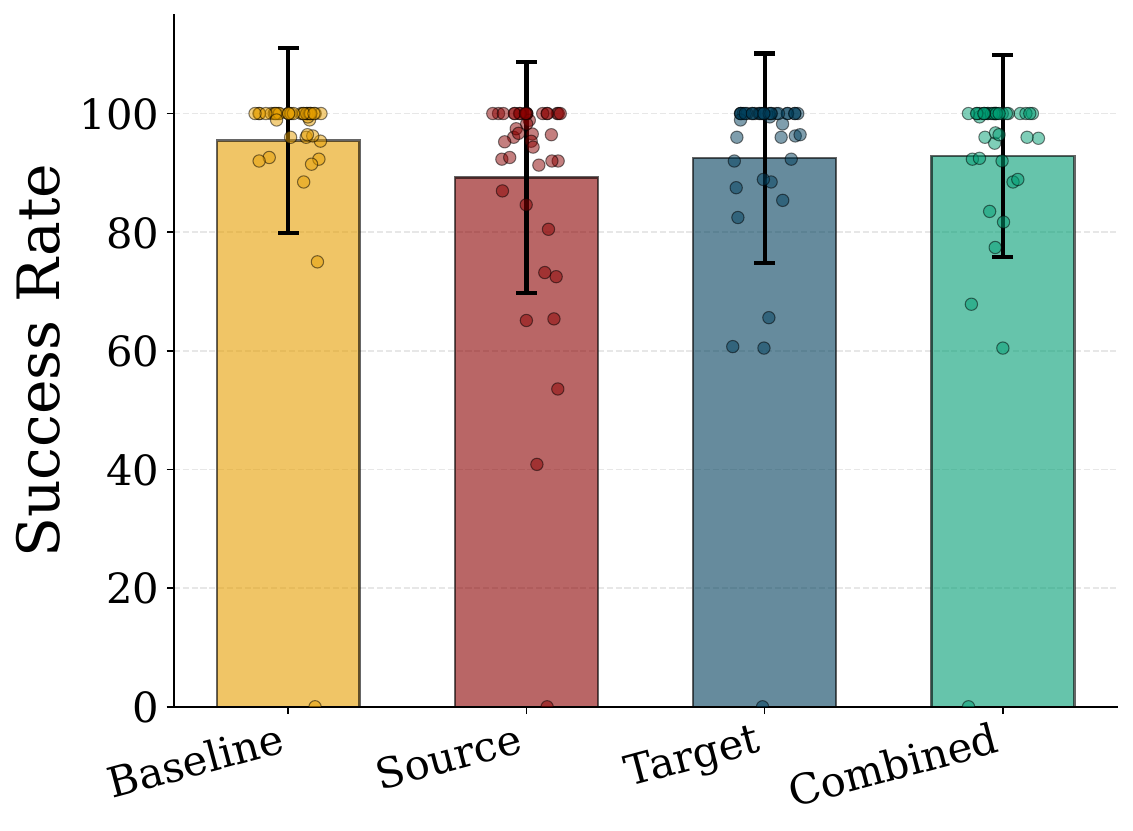}
    \caption{Effectiveness}
  \end{subfigure}\hfill
  \begin{subfigure}{0.5\columnwidth}
    \centering
    \includegraphics[width=0.91\linewidth]{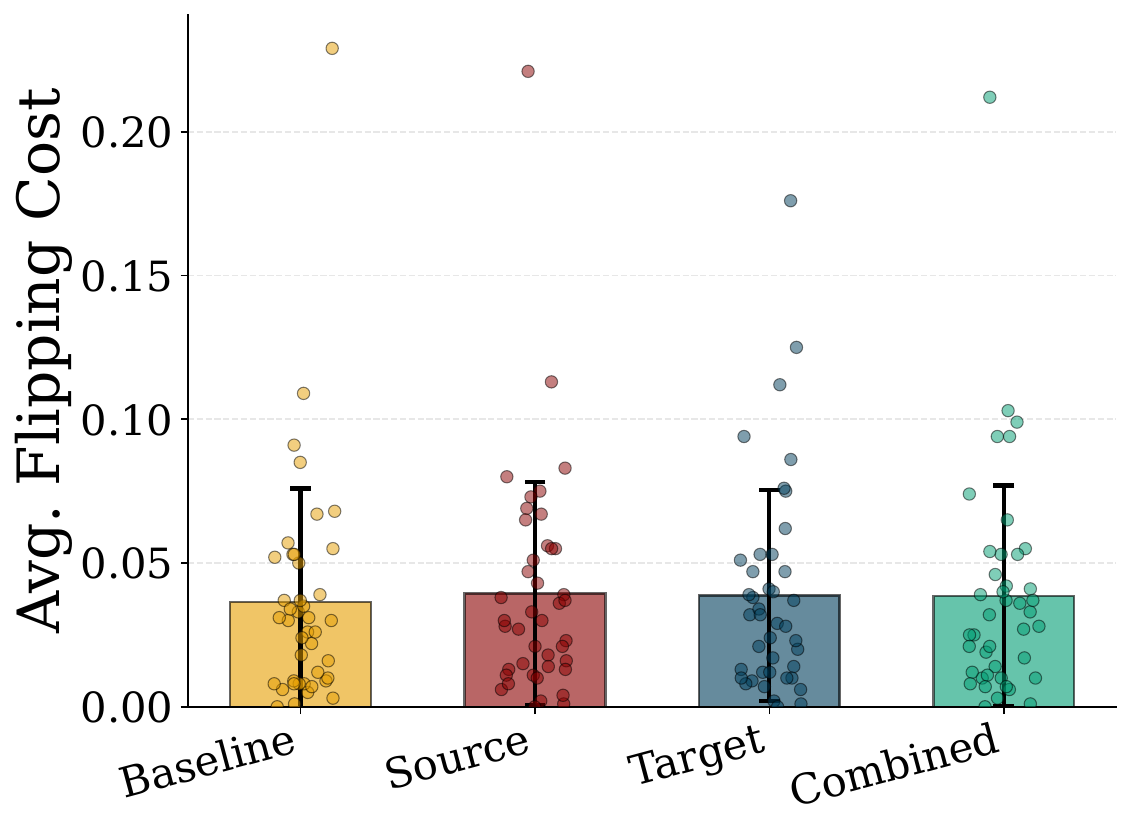}
    \caption{Avg. Flipping Cost}
  \end{subfigure}
  \begin{subfigure}{0.5\columnwidth}
    \centering
    \includegraphics[width=0.91\linewidth]{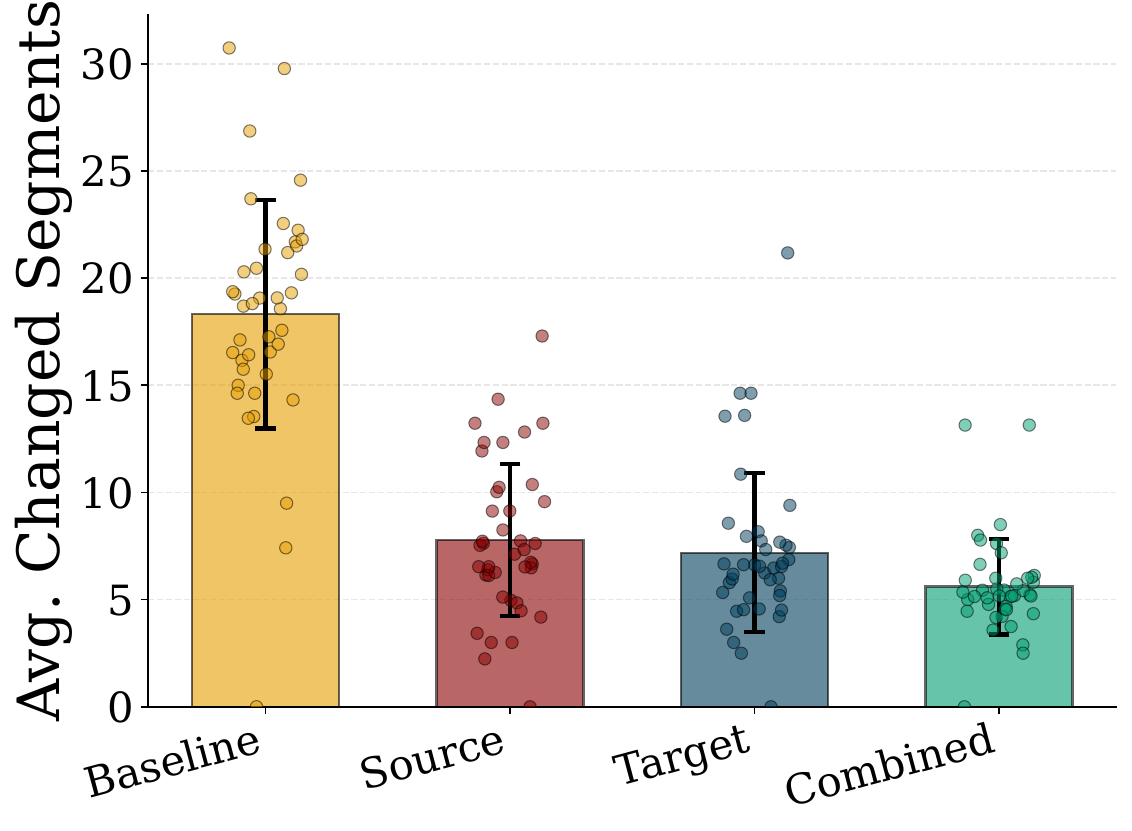}
    \caption{Avg. Changed Segments}
  \end{subfigure}\hfill
  \begin{subfigure}{0.5\columnwidth}
    \centering
    \includegraphics[width=0.91\linewidth]{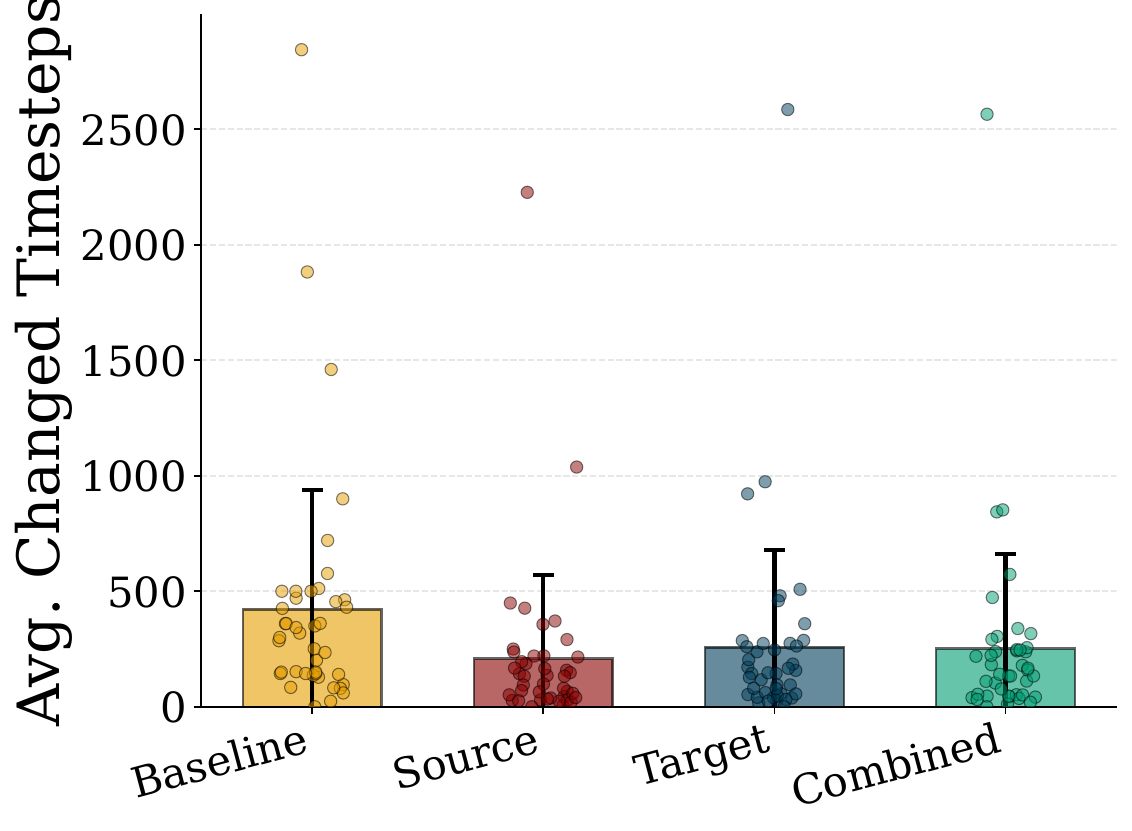}
    \caption{Avg. Changed Timesteps}
  \end{subfigure}
  \caption{Importance Weighting Strategies Comparison. Points denote individual datasets; bars indicate dataset-wide averages and standard deviations for validity and sparsity.}
  \label{fig:weighting_results}
\end{figure}

\begin{table*}
\centering
\caption{
Local Performance Benchmarking. Comparative analysis of \textsc{Galactic-L} against \textsc{kNN}, \textsc{TSEvo}, and \textsc{Glacier} variants across effectiveness, average flipping cost, average segments changed, average timesteps changed, and runtime. The null indicator ``-'' is assigned throughout the metrics where the corresponding effectiveness is zero. Results for additional datasets can be found in Appendix \Cref{tab:additional_local_results}.
}
\tiny
\setlength{\tabcolsep}{4pt}
\begin{tabular}{@{} l ccccc  ccccc  ccccc  ccccc  ccccc @{}}
\toprule
\multirow{2}{*}{Dataset} & \multicolumn{5}{c}{\textsc{KNN}} & \multicolumn{5}{c}{\textsc{TSEvo}} & \multicolumn{5}{c}{\textsc{Glacier-L}} & \multicolumn{5}{c}{\textsc{Glacier-G}} & \multicolumn{5}{c}{\textsc{Galactic-L}} \\
\cmidrule(lr){2-6} \cmidrule(lr){7-11} \cmidrule(lr){12-16} \cmidrule(lr){17-21} \cmidrule(lr){22-26}
 & eff& afc & asc & atc & RT &  eff& afc & asc & atc & RT & eff& afc & asc & atc & RT & eff& afc & asc & atc & RT &  eff& afc & asc & atc & RT \\
\midrule
ACSF1 & 100 & 0.08 & 24.82 & 1460 & 0.15 & 100 & 0.09 & 24.82 & 1460 & 405.54 & 28.33 & 0.01 & 21.88 & 1460 & 9.69 & 26.67 & 0.01 & 21.94 & 1459.94 & 9.95 & 98.33 & 0.03 & 7.75 & 892.08 & 22.12 \\
ArrowHead & 100 & 0.1 & 35.79 & 251 & 0.17 & 100 & 0.09 & 35.79 & 251 & 385.22 & 36.51 & 0 & 34.22 & 251 & 11.06 & 52.38 & 0.01 & 33.73 & 251 & 18.27 & 98.41 & 0.02 & 9.47 & 131.39 & 11.21 \\
Car & 100 & 0.12 & 27.61 & 577 & 0.07 & 100 & 0.17 & 27.61 & 577 & 240.77 & 69.44 & 0.01 & 27.44 & 577 & 12.04 & 88.89 & 0.01 & 27.25 & 576.97 & 15.22 & 100 & 0.03 & 7.36 & 318.06 & 5.13 \\
CBF & 100 & 0.16 & 26.49 & 128 & 1.34 & 100 & 0.19 & 26.49 & 128 & 1751.15 & 0.36 & 0.02 & 28 & 128 & 0.64 & 0.36 & 0.02 & 28 & 128 & 0.63 & 76.7 & 0.09 & 8.55 & 53.03 & 158.08 \\
Coffee & 100 & 0.24 & 14.76 & 286 & 0.05 & 100 & 0.21 & 14.76 & 286 & 105.57 & 29.41 & 0.06 & 15.6 & 286 & 3.14 & 58.82 & 0.06 & 15.3 & 286 & 5.99 & 94.12 & 0.18 & 4 & 119.75 & 4.82 \\
ECG200 & 0 & - & - & - & 0.15 & 100 & 0.12 & 19.38 & 96 & 340.19 & 33.33 & 0.01 & 20 & 96 & 10.95 & 45 & 0.01 & 20.07 & 96 & 23.83 & 95 & 0.04 & 5.25 & 47.11 & 32.13 \\
Fish & 100 & 0.1 & 29.45 & 463 & 0.28 & 100 & 0.13 & 29.45 & 463 & 1071.66 & 30.48 & 0.01 & 29.22 & 462.94 & 25.33 & 40 & 0.01 & 29.5 & 463 & 29.91 & 99.05 & 0.04 & 7.85 & 258.37 & 23.2 \\
SyntheticControl & 100 & 0.21 & 14.71 & 60 & 0.27 & 100 & 0.23 & 14.71 & 60 & 570.44 & 1.11 & 0.02 & 17 & 60 & 1.32 & 0 & - & - & - & 14.91 & 91.67 & 0.1 & 5.01 & 24.73 & 199.21 \\
UMD & 0 & - & - & - & 0.18 & 100 & 0.18 & 25.59 & 149.8 & 317.75 & 18.52 & 0.02 & 27.6 & 150 & 5.02 & 25.93 & 0.03 & 27.43 & 150 & 12.36 & 98.15 & 0.11 & 7.04 & 80.92 & 30 \\
Worms & 100 & 0.18 & 37.24 & 900 & 0.43 & 100 & 0.15 & 37.24 & 900 & 499.46 & 36.76 & 0 & 36.56 & 900 & 14.59 & 41.18 & 0 & 36.39 & 899.93 & 14.9 & 98.53 & 0.01 & 9.79 & 464.4 & 8.67 \\
\bottomrule
\end{tabular}
\label{tab:baseline_results}
\end{table*}

\begin{table*}
\centering
\caption{Global Performance Benchmarking. Comparative analysis of \textsc{Galactic-G} variants against \textsc{GLOBE-CE}$^*$ and \textsc{Glacier-G}$^*$ across effectiveness, average flipping cost, average segments changed, average timesteps changed, and runtime. The $^*$ indicates adaptations for the global timeseries clustering context; the null indicator ``-'' is assigned throughout the metrics where the corresponding effectiveness is zero. Results for additional datasets can be found in Appendix \Cref{tab:additional_global_results}.}
\label{tab:global_results}
\resizebox{\textwidth}{!}{%
\begin{tabular}{lcccccccccccccccccccccccccccccc}
\toprule
& \multicolumn{20}{c}{\textsc{Galactic-G}} & \multicolumn{10}{c}{\textsc{Competitors}} \\
\cmidrule(lr){2-21} \cmidrule(lr){22-31}
 Dataset & \multicolumn{5}{c}{\textsc{Optimal}} & \multicolumn{5}{c}{\textsc{Greedy}} & \multicolumn{5}{c}{\textsc{Hierarchical}} & \multicolumn{5}{c}{\textsc{Hierarchical Greedy}} & \multicolumn{5}{c}{\textsc{GLOBE-CE$^*$}} & \multicolumn{5}{c}{\textsc{Glacier-G$^*$}}\\
\cmidrule(lr){2-6} \cmidrule(lr){7-11} \cmidrule(lr){12-16} \cmidrule(lr){17-21} \cmidrule(lr){22-26} \cmidrule(lr){27-31}
 
  & eff& afc & asc & atc & RT &  eff& afc & asc & atc & RT & eff& afc & asc & atc & RT & eff& afc & asc & atc & RT &  eff& afc & asc & atc & RT &  eff& afc & asc & atc & RT \\
 
\midrule
ACSF1 & 62.8 & 0.031 & 3.7 & 813 & 4.17 & 62.8 & 0.021 & 3.7 & 812 & 0.66 & 37.6 & 0.010 & 3.4 & 591 & 1.49 & 37.6 & 0.009 & 3.2 & 581 & 0.89 & 21.2 & 0.006 & 4.0 & 730 & 60.93 & 31.2 & 0.006 & 5.4 & 1022 & 0.01 \\
ArrowHead & 75.9 & 0.021 & 5.0 & 102 & 49.55 & 75.9 & 0.019 & 4.7 & 102 & 5.34 & 72.0 & 0.021 & 4.9 & 97 & 3.34 & 72.0 & 0.019 & 4.6 & 102 & 2.80 & 54.6 & 0.088 & 9.4 & 251 & 35.57 & 20.2 & 0.006 & 7.9 & 167 & 10.30 \\
Car & 70.8 & 0.024 & 4.8 & 350 & 106.68 & 70.8 & 0.023 & 4.8 & 361 & 3.18 & 55.2 & 0.015 & 4.4 & 282 & 3.24 & 58.6 & 0.015 & 4.2 & 280 & 2.41 & 74.2 & 0.056 & 6.7 & 433 & 32.16 & 33.4 & 0.008 & 8.8 & 577 & 0.00 \\
CBF & 39.4 & 0.105 & 3.3 & 47 & 36.93 & 39.4 & 0.105 & 3.3 & 47 & 10.21 & 33.7 & 0.098 & 3.3 & 51 & 1.57 & 33.7 & 0.094 & 3.3 & 51 & 1.65 & 26.0 & 0.186 & 7.1 & 128 & 50.57 & 0.0 & - & - & - & 0.00 \\
Coffee & 100.0 & 0.186 & 4.3 & 107 & 3.32 & 100.0 & 0.177 & 4.4 & 114 & 2.15 & 100.0 & 0.186 & 4.3 & 107 & 1.59 & 100.0 & 0.177 & 4.4 & 114 & 1.59 & 43.1 & 0.436 & 4.9 & 143 & 26.52 & 52.4 & 0.074 & 9.5 & 286 & 7.25 \\
ECG200 & 58.8 & 0.056 & 4.0 & 42 & 287.92 & 58.8 & 0.056 & 4.0 & 42 & 17.93 & 53.5 & 0.048 & 4.0 & 40 & 12.11 & 57.6 & 0.053 & 4.0 & 41 & 8.48 & 59.8 & 0.272 & 8.4 & 96 & 22.13 & 6.2 & 0.018 & 13.2 & 96 & 25.75 \\
Fish & 87.4 & 0.040 & 5.5 & 286 & 5929.39 & 87.4 & 0.034 & 5.2 & 245 & 18.64 & 67.6 & 0.035 & 5.8 & 268 & 8.80 & 68.8 & 0.033 & 5.3 & 258 & 4.77 & 82.5 & 0.109 & 7.2 & 397 & 49.87 & 31.1 & 0.011 & 9.5 & 463 & 27.03 \\
SyntheticControl & 72.3 & 0.115 & 3.3 & 19 & 136.92 & 72.3 & 0.106 & 3.3 & 19 & 23.68 & 69.5 & 0.110 & 3.5 & 19 & 20.21 & 69.7 & 0.107 & 3.6 & 20 & 18.03 & 68.0 & 0.280 & 6.5 & 50 & 29.96 & 0.0 & - & - & - & 14.91 \\
UMD & 84.2 & 0.104 & 6.4 & 75 & 1491.24 & 84.2 & 0.093 & 6.0 & 70 & 18.89 & 81.9 & 0.102 & 6.6 & 80 & 14.02 & 81.9 & 0.091 & 6.1 & 72 & 9.72 & 66.1 & 0.059 & 9.1 & 100 & 37.35 & 14.4 & 0.020 & 11.5 & 100 & 19.34 \\
Worms & 61.0 & 0.012 & 7.5 & 427 & 213.45 & 61.4 & 0.012 & 8.0 & 437 & 5.17 & 51.9 & 0.011 & 8.1 & 447 & 10.94 & 54.0 & 0.011 & 8.1 & 437 & 8.51 & 35.1 & 0.019 & 16.6 & 900 & 57.71 & 20.4 & 0.006 & 21.0 & 900 & 0.00 \\
\bottomrule
\end{tabular}
}
\end{table*}

\subsection{Local Evaluation}
\label{sec:baselines}
We benchmark the \textsc{Galactic-L} against three state-of-the-art baselines that represent different paradigms of counterfactual search: 
\textsc{$k$-Nearest Neighbors ($k$-NN)} (retrieval), \textsc{TSEvo}~\cite{hollig2022tsevo} (evolutionary optimization), and \textsc{Glacier}~\cite{wang2024glacier} (gradient-based constrained search), evaluating both its \textsc{Glacier-Local} and \textsc{Glacier-Global} variants. Implementation details and method-specific configurations are reported in ~\Cref{app:local_baselines}.
\Cref{tab:baseline_results} presents the summarized results of all competing methods across 10 UCR datasets; results for an additional 20 datasets are provided in Appendix \Cref{tab:additional_local_results}.

The retrieval-based \textsc{$k$-NN} baseline serves as an informative lower bound: although it achieves high success rates when nearby cross-cluster instances exist, it consistently incurs high perturbation cost and poor sparsity, as retrieved neighbors often differ from the query across large portions of the series, with effectiveness degrading in sparse regions of the data. 
\textsc{TSEvo} attains high validity due to its exploratory black-box evolutionary search, but at the expense of substantially higher flipping costs, extensive multi-segment edits, and orders-of-magnitude higher runtimes, yielding less localized and harder to interpret explanations. \textsc{Glacier} improves upon black-box methods in both cost and efficiency through gradient-based, importance-aware constraints; however, its reliance on soft penalties allows perturbations to leak into non-critical regions, often resulting in more modified segments and timesteps than necessary, particularly for the global variant, where a single mask is applied across heterogeneous cluster members.

In contrast, \textsc{Galactic-L} achieves the most favorable trade-off across all dimensions, exceeding the effectiveness of gradient-based baselines while producing substantially sparser, low-cost, and cluster-specific counterfactuals.
By restricting optimization to transition-critical temporal regions per cluster, \textsc{Galactic-L} preserves the global structure of the series and offers explanations that are both highly interpretable and computationally efficient, being orders 
of magnitude faster than evolutionary methods and consistently faster than gradient-based approaches.

\subsection{Global Evaluation}
\label{sec:global_exp}
This experiment evaluates the quality and computational efficiency of the global selection phase of our four \textsc{Galactic-G} variants (Optimal, Greedy, and their respective Hierarchical extensions) against two state-of-the-art global counterfactual baselines adapted for the clustering temporal domain. In the absence of native global generators for time-series clustering, we evaluate against
\textsc{Glacier-G}$^*$, an adaptation of the \textsc{Glacier} framework~\cite{wang2024glacier}, where local perturbations are aggregated from all correctly clustered instances via $k$-means clustering to extract representative centroids, and \textsc{GLOBE-CE}$^*$, an adaptation of \textsc{GLOBE-CE}~\cite{ley2023globe}, that flattens sequences into high-dimensional vectors to optimize global directions through random sampling, and uses bisection-based scalar optimization to maximize effectiveness. 
Table~\ref{tab:global_results} presents summarized results for all competing methods across 10 UCR datasets; results for an additional 20 datasets are provided in Appendix \Cref{tab:additional_global_results}. 

\textsc{GLOBE-CE$^*$} relies on a tabularized view of the series that ignores temporal dependencies, resulting in global deltas that lack structural sparsity and generalize poorly. 
Its optimization strategy further amplifies this limitation: it generates large sets of random translation vectors and applies bidirectional scaling to increase coverage, which introduces substantial computational overhead compared to \textsc{Galactic-G} variants and \textsc{Glacier-G$^*$}.
Since this process is not guided by the temporal structure or the cluster manifold, it often requires large-magnitude shifts to capture outlying instances, ultimately leading to inflated average flipping costs.

The \textsc{Glacier-G$^*$} approach struggles to provide adequate effectiveness at the cluster level. Aggregating individually generated deltas does not sufficiently capture the underlying cluster structure, resulting in global directions that remain instance-specific and fail to represent coherent, cluster-level transformations.

\textsc{Galactic-G} consistently outperforms adapted baselines, achieving a superior effectiveness-to-cost trade-off with significantly lower computational latency. The MDL-based selection identifies the most representative and ``compressible'' perturbations. This strategy respects the internal segmentation of the cluster, resulting in global explanations that maximize population coverage, while minimizing both the number of modified segments and total flipping cost. Consequently, \textsc{Galactic-G} provides summaries that are both more interpretable and computationally efficient than traditional aggregation or random-search methods.

%% file: sections/conclusions.tex
\section{Conclusions}
\label{sec:conclusions}
In this work, we present \textsc{Galactic}, a unified framework bridging the gap between local and global explainability in time-series clustering. By rigorously formulating the explanation selection as an MDL-based submodular maximization problem, we provide a theoretically grounded approach to summarize cluster transitions. Extensive empirical evaluation demonstrates that \textsc{Galactic} consistently outperforms state-of-the-art baselines in sparsity and coherence while maintaining superior computational efficiency. 

Future work will extend this formalism to multivariate settings and integrate causal discovery to disentangle latent temporal confounders. Furthermore, we intend to explore how the MDL-based adaptation can be generalized to the tabular data domain for broader explainability applications.

%% file: sections/appendix.tex
\section{Notation}
A summary of the notation used throughout the paper is provided in \Cref{tab:notation}.
\begin{table}[h]
    \centering
    \caption{Notation Table for the \textsc{Galactic} Framework}
    \label{tab:notation}
\small
    \setlength{\tabcolsep}{4pt}
    \begin{tabular}{ll}
         \toprule
         \textbf{Symbol} & \textbf{Meaning}\\
         \midrule
         $\mathcal{T}, T$ & Space of series $\mathbb{R}^T$; Series length \\
         $\vec x = [x_1, \dots, x_T]$ & A time-series instance of length $T$ \\
         $X = \{x^{(i)}\}_{i=1}^N$ & Dataset containing $N$ series \\
         $f: \mathcal{T} \to \{1, \dots, K\}$ & Underlying clustering algorithm \\
         $\mathcal{C} = \{C_1, \dots, C_K\}$ & Cluster partition \\
         $g: \mathcal{T} \to [0,1]^K$ & Probabilistic surrogate classifier \\
         $\mathcal{S}(x)$ & Set of $M$ contiguous, disjoint temporal segments \\
         $x^{cf}, \delta$ & Local counterfactual $x^{cf} = x + \delta$; Perturbation $\delta$ \\
         % $k$ & Source cluster index\\
         % $ k'$ & Target cluster index ($k' \neq k$) \\
         $\text{cost}(\cdot)$ & Perturbation cost metric (e.g., $\ell_p$ distance) \\
         $\Delta_k$ & Candidate perturbations for cluster $C_k$ \\
         \midrule
         $p\_sz$ & The bit-cost of a pointer \\
         $M$ & The explanation model \\
         $D$ & The data population to be explained \\
          $L(M):= L(\mathbb S)$ & Model description length \\
         $L(D \mid M):= L(C_k \mid \mathbb S)$ & Data description length \\
         $L(D, M) := L(C_k, \mathbb S)$ & MDL objective: $L(M) + L(D \mid M)$ \\
         $\mathrm{cov}(C_k, \mathbb S)$ & Set of instances in $C_k$ successfully flipped by $\mathbb S$ \\
         $\mathrm{eff}(C_k, \mathbb S)$ & Effectiveness of the summary $\mathbb S$ \\
         $\mathrm{cost}_{\text{flip}}(\vec x, \mathbb S)$ & Minimum cost to flip instance $x$ using the set $\mathbb S$ \\
         $\mathrm{afc}(C_k, \mathbb S)$ & Average Flipping Cost\\
         $\Delta(\mathbb S)$ & MDL reduction: $L(D, \varnothing) - L(D, \mathbb S)$ \\
\bottomrule
    \end{tabular}
\end{table}

\section{Theoretical Foundations of Global Selection}
\label{sec:appendix_proofs}
The global selection phase of \textsc{Galactic} aims to identify a non-redundant set of representative perturbations that effectively summarize cluster transitions. As delineated in Section \ref{prob:globalMDL}, this is framed as a Minimum Description Length (MDL) optimization problem. Given that the exhaustive search over the power set of all candidate perturbations is NP-hard, we employ a greedy approximation algorithm. To guarantee the theoretical bounds of this greedy approach, we provide here the formal proof that the MDL objective function exhibits \textit{supermodularity}, and that the MDL reduction is \textit{submodular}.
For these proofs, we make the following assumption.
\begin{assumption}[Admissible Additions]
\label{ass:admissible}
When analyzing the addition of a perturbation $\vec{\delta} \in  \Delta_k$, to our model
$\vec{\delta} \notin \mathbb S$,
we only allow admissible additions that newly cover at least one
previously uncovered instance, i.e.,
\(
\mathrm{cov}(C_k,\mathbb S\cup\{\vec{\delta}\})
\setminus
\mathrm{cov}(C_k,\mathbb S)
\neq \emptyset.
\)
\end{assumption}

\input{sections/proofs}

\section{Algorithms}
\label{app:algos}

\begin{algorithm}
\caption{\textsc{Galactic-G}: Optimal}
\label{alg:optimal}
\small
\begin{algorithmic}[1]
\Require Cluster $C_k$, groups $\mathcal{G}_k$, budget $\mu_k$ 
\Ensure Global perturbation set $\mathbb{S}_{C_k}$
\Procedure{Optimal}{$\Delta, C, \mu$}
\State \Return $\arg\min_{\mathbb{S} \subseteq \Delta, |\mathbb{S}| \le \mu} L(C, \mathbb{S})$ \Comment{Exhaustive search for small $|\Delta|$}
\EndProcedure
\end{algorithmic}
\end{algorithm}

\begin{algorithm}
\caption{\textsc{Galactic-G}: Greedy}
\label{alg:greedy}
\small
\begin{algorithmic}[1]
\Require Cluster $C_k$, groups $\mathcal{G}_k$, budget $\mu_k$ 
\Ensure Global perturbation set $\mathbb{S}_{C_k}$
\Procedure{Greedy}{$\Delta, C, \mu$}
\State $\mathbb{S} \gets \emptyset$
\While{$|\mathbb{S}| < \mu$}
\State $\vec{\delta}^* \gets \arg\min_{\vec{\delta} \in \Delta \setminus \mathbb{S}} \left( L(\mathbb{S} \cup \{\vec{\delta}\}) + L(C \mid \mathbb{S} \cup \{\vec{\delta}\}) \right)$
\If{$L(\mathbb{S} \cup \{\vec{\delta}^*\}) + L(C \mid \mathbb{S} \cup \{\vec{\delta}^*\}) \ge L(\mathbb{S}) + L(C \mid \mathbb{S})$}
\State \textbf{break} \Comment{Stop if the joint description length ceases to decrease}
\EndIf
\State $\mathbb{S} \gets \mathbb{S} \cup \{\vec{\delta}^*\}$
\EndWhile
\State \Return $\mathbb{S}$
\EndProcedure
\end{algorithmic}
\end{algorithm}

\section{Technical Implementation Details}
\label{sec:technical_details}
This section provides a technical report of the model specifications and optimization trajectories required for reproducible local search, alongside the structural mechanics of temporal segmentation for subgroup discovery. Furthermore, we formalize the target selection policies and the mathematical formulation of multi-level importance weighting used to enforce temporal consistency and discriminative precision during counterfactual generation.

\subsection{Model Specifications}
As surrogate, we use a deep 1D convolutional residual network featuring three blocks with kernel sizes $\{8, 5, 3\}$. Each block employs batch normalization, ReLU activation, and shortcut connections to preserve gradient flow. Global Average Pooling (GAP) aggregates features before the final dense projection.

\subsection{Local Search and Reproducibility}
The Adam optimizer utilizes a step size of $0.01$ and a stagnation threshold at $0.005$ to filter negligible perturbations. All experiments were conducted on an AMD Ryzen 9 5950X CPU and an NVIDIA RTX A6000 GPU. For the local baseline comparison, we explain a random sample of $30\%$ of correctly classified instances per cluster to ensure a statistically representative evaluation of local fidelity.

\paragraph{Implementation Details.} Temporal sequences are partitioned using \textsc{NNSegment} with a window size of $\lfloor T/10 \rfloor$ and a step size of $\lfloor T/6.67 \rfloor$. We employ K-medoids clustering for group formation, with the optimal cluster count determined via the Silhouette coefficient. Local search utilizes the Adam optimizer ($max\_iter=500$, $patience=2$). For our MDL objective, we assume a standard 64-bit pointer size and compute log-likelihoods as $\log_2(\cdot + 1)$ to preserve monotonicity.
Based on the empirical analysis of local counterfactual robustness (\Cref{sec:target_policy_main}, \Cref{sec:target_selection_appendix}), we adopt the \textit{second possible} cluster—i.e., the most probable alternative under the surrogate—as the default target direction for \textsc{Galactic-L} and \textsc{Galactic-G}. We further employ the \textit{combined} importance weighting strategy, integrating both source- and target-cluster discriminative regions. This configuration consistently yielded higher effectiveness and more stable sparsity–cost trade-offs across comparisons (\Cref{sec:importance_weighting_main}) and is therefore used throughout all subsequent local and global experiments.

\subsection{Target Selection Policies and Evaluation Statistics}
\label{sec:target_policies}
We consider three policies for selecting (or not selecting) a target cluster during counterfactual optimization:
(i) \textit{second possible}, which selects the cluster with the second-highest posterior probability under the surrogate model;
(ii) \textit{random}, which randomly selects a target cluster different from the source and explicitly optimizes toward that cluster; and
(iii) \textit{all\_random}, which does not predefine a target cluster and instead drives the optimization by pushing the instance away from its current cluster, allowing the search to implicitly discover a new assignment.

Given two policies $x$ and $y$, and a metric $m$ evaluated over a set of datasets $\mathcal{D}$, we define the mean signed difference:
\[
\Delta_m = \frac{1}{|\mathcal{D}|} \sum_{d \in \mathcal{D}} \left( m_d^{(y)} - m_d^{(x)} \right).
\]
We refer to this quantity as a \emph{Gain} when it improves the metric in the desirable direction (higher is better for success rate; lower is better for cost, sparsity, and runtime), and as a \emph{Loss} otherwise.

\subsection{Temporal Segmentation and Subgroup Discovery}
\label{sec:strategies}
\textsc{Galactic} builds on the premise that cluster identity in time series is defined by distinctive temporal topologies. However, unsupervised clusterings often exhibit internal variability, including instances not strictly conforming to a single centroidal pattern (i.e., criticisms \cite{gretton2012kernel}). To preserve interpretability while accounting for heterogeneity, we adopt a multi-modal strategy capturing dominant sub-patterns within each cluster $C_k$.

\paragraph{Instance-wise Partitioning.}
For each series $x \in C_k$, we compute a segmentation
$\mathcal{S}(x) = \{I_1, I_2, \ldots, I_M\}$,
where each segment
$I_m = [\tau_{m-1}+1, \ldots, \tau_m]$
is a contiguous interval defined by breakpoints
$0 = \tau_0 < \tau_1 < \cdots < \tau_M = T$.
Segmentations are obtained via \textsc{NNSegment} \cite{sivill2022limesegment}, which identifies meaningful change-points by detecting local deviations in the signal.

\paragraph{Subgroup Resolution.}
To uncover structurally distinct temporal regimes within a cluster, we partition the set of gap vectors
$\{\gamma(x) : x \in C_k\}$
using $K_{\text{seg}}$-medoids clustering.
The number of subgroups is selected by maximizing the Silhouette score.
To ensure that the resulting structure is meaningful, we retain only solutions in which all subgroups contain at least two instances; if no such solution exists, we fall back to the highest-scoring partition even if singleton groups are present.
The resulting medoids
$\mathcal{P}_k = \{\mathcal{S}_1, \ldots, \mathcal{S}_R\}$
define the \emph{dominant segmentation regimes} of cluster $C_k$, compactly summarizing its internal temporal variability.

\subsection{Multi-level Importance Weighting}
Given the dominant segmentation regimes, we quantify the contribution of specific temporal segments to cluster identity, ensuring that counterfactual perturbations are directed toward structurally influential regions.
For an instance $x$ belonging to a subgroup with segmentation $\mathcal{S}_r$, the importance of each segment $I_m \in \mathcal{S}_r$ is computed via block-wise permutation:
\begin{equation}
\operatorname{imp}(I_m) =
\left|
a_c - \frac{1}{B} \sum_{b=1}^{B} a_{c,m}^{(b)}
\right|,
\end{equation}
where $a_c$ denotes the surrogate accuracy for the corresponding cluster and
$a_{c,m}^{(b)}$ is the accuracy obtained after permuting the values of segment $I_m$ in the $b$-th perturbation.
Segment-level scores are then propagated to a timestep-level mask $w(x) \in [0,1]^T$.

Figure~\ref{fig:segment_weightings} illustrates the resulting importance masks under different weighting strategies.
Depending on the optimization objective, we consider three complementary schemes:

\paragraph{1. Source-Cluster Importance ($w^{(s)}$).}
This strategy prioritizes preserving the defining temporal structure of the source cluster $C_s$.
Segment importances are binarized using a quantile threshold,
$w^{(s)}_t = \mathbb{1}\{\operatorname{imp}^{(s)}(I_m) \geq Q_{0.75}\}$,
effectively preventing the optimizer from perturbing regions that encode the original cluster identity.
An example of this behavior is shown in Fig.~\ref{fig:segment_weightings}(b), where the counterfactual search explicitly protects perturbing regions that define the original identity.

\paragraph{2. Target-Cluster Importance ($w^{(t)}$).}
This strategy guides perturbations toward regions that characterize the target cluster $C_t$.
To accommodate the multi-modal structure of $C_t$, we align all segmentation breakpoints across all subgroup sets of $G_t$ into a unified collection of intervals $[a,b)$.
The importance of each unified interval is computed as a weighted average across subgroups:
\begin{equation}
\operatorname{imp}_{[a,b)} =
\frac{1}{|G_t|}
\sum_{g \in G_t}
\sum_{I \in \mathcal{S}_g}
\operatorname{imp}_g(I)
\cdot
\frac{|[a,b)\cap I|}{|I|}.
\end{equation}
This strategy, visualized in Fig.~\ref{fig:segment_weightings}(c), emphasizes temporal regions that must be modified to achieve membership in the target cluster.

\paragraph{3. Combined Importance ($w^{(s,t)}$).}
The combined strategy simultaneously enforces the structural constraints of the source cluster and the discriminative requirements of the target cluster.
We form the unified group set
$G_{s,t} = \{g_s\} \cup G_t$,
where $g_s$ is the source-group segmentation of $x$.
After aligning all breakpoints from $G_{s,t}$ into unified intervals $[a,b)$, the joint importance is computed as:
\begin{equation}
\operatorname{imp}_{[a,b)} =
\frac{1}{|G_{s,t}|}
\sum_{g \in G_{s,t}}
\sum_{I \in \mathcal{S}_g}
\operatorname{imp}_g(I)
\cdot
\frac{|[a,b)\cap I|}{|I|}.
\end{equation}
This strategy is illustrated in Fig.~\ref{fig:segment_weightings}(d), where the mask highlights intervals that are simultaneously critical for preserving the source identity and achieving the target assignment.

\section{Local Baselines and Configurations}
\label{app:local_baselines}
In this section, we provide implementation details and configuration choices for all local counterfactual baselines evaluated in \Cref{sec:baselines}.

\paragraph{Baselines.} 
\textsc{$k$-Nearest Neighbors ($k$-NN)}: This retrieval-based baseline represents the ``least effort'' strategy for generating a counterfactual explanation. For a series $x$, we identify the $k$ closest series in the training set belonging to different clusters. A successful counterfactual is found if any of these neighbors are correctly classified into the target cluster $C_t$ by the surrogate model. If successful, no further optimization is performed.
\textsc{TSEvo}~\cite{hollig2022tsevo}: A multi-objective evolutionary approach that treats the surrogate model as a black-box. It does not require gradient access and generates counterfactuals through genetic operations. Specifically, it utilizes \textit{mutation} (applying transformations such as noise, interpolation, and shifting) and \textit{crossover} (recombining segments from different series) to navigate the search space. 
\textsc{Glacier}~\cite{wang2024glacier}: The closest baseline to our framework, which follows a gradient-driven search guided by local constraints. Glacier identifies segments and assigns importance scores to discourage changes in certain regions via a binary constraint vector $c$ in its loss function. We evaluate two variants: \textsc{Glacier-Local} (example-specific constraints based on individual explanations) and \textsc{Glacier-Global} (applying a common importance mask derived from the whole dataset).

\paragraph{Experimental Configuration and Reproducibility.} 
For a fair comparison, we standardized the search conditions across all methods. 
For TSEvo, we capped the number of generations at $10$ (default: $500$), as the default was computationally infeasible for large-scale evaluation on the UCR Archive. Preliminary runs showed that effectiveness stabilized early, with $10$ generations already achieving high flipping success while ensuring tractable and reproducible runtimes.
For Glacier, we utilized their default hyperparameters but opted for the variant without autoencoders (\textit{NoAE}). This ensures that the generated counterfactuals are evaluated purely on the $L_2$ flipping cost measure without being biased by the generative preference of the model for ``realistic'' data distributions, which can intentionally increase edit costs. 
All methods are evaluated using the same surrogate model and afc ($L_2$).

\section{Extended Experimental Results}
This section provides a supplementary analytical deep-dive into the empirical performance of the \textsc{Galactic} framework, offering an additional exploration of the hyperparameter space and algorithmic configurations that govern counterfactual generation. We extend the core evaluations presented in the main manuscript by dissecting the nuances of target selection policies and the sensitivity of the global selection mechanism \textsc{Galactic-G} to the perturbation budget $\mu$. Furthermore, we provide granular, per-dataset baseline comparisons for both local and global explanations, ensuring a robust validation of the superiority of our framework across diverse time-series topologies and clustering manifolds.

\begin{figure}
  \centering
  \begin{subfigure}{0.5\columnwidth}
    \includegraphics[height=0.8\linewidth]{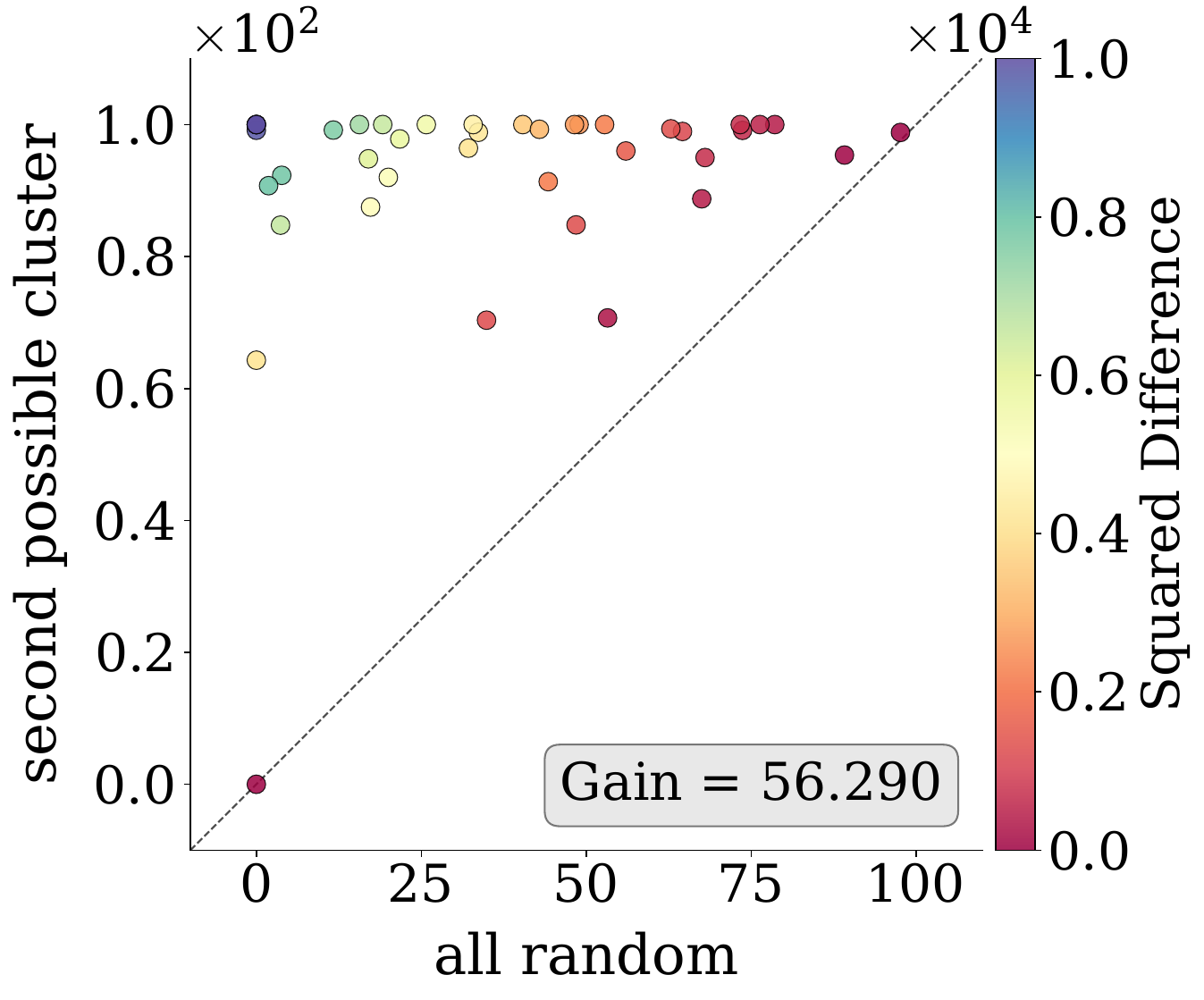}
    \caption{Effectiveness}
  \end{subfigure}\hfill
  \begin{subfigure}{0.5\columnwidth}
    \includegraphics[height=0.8\linewidth]{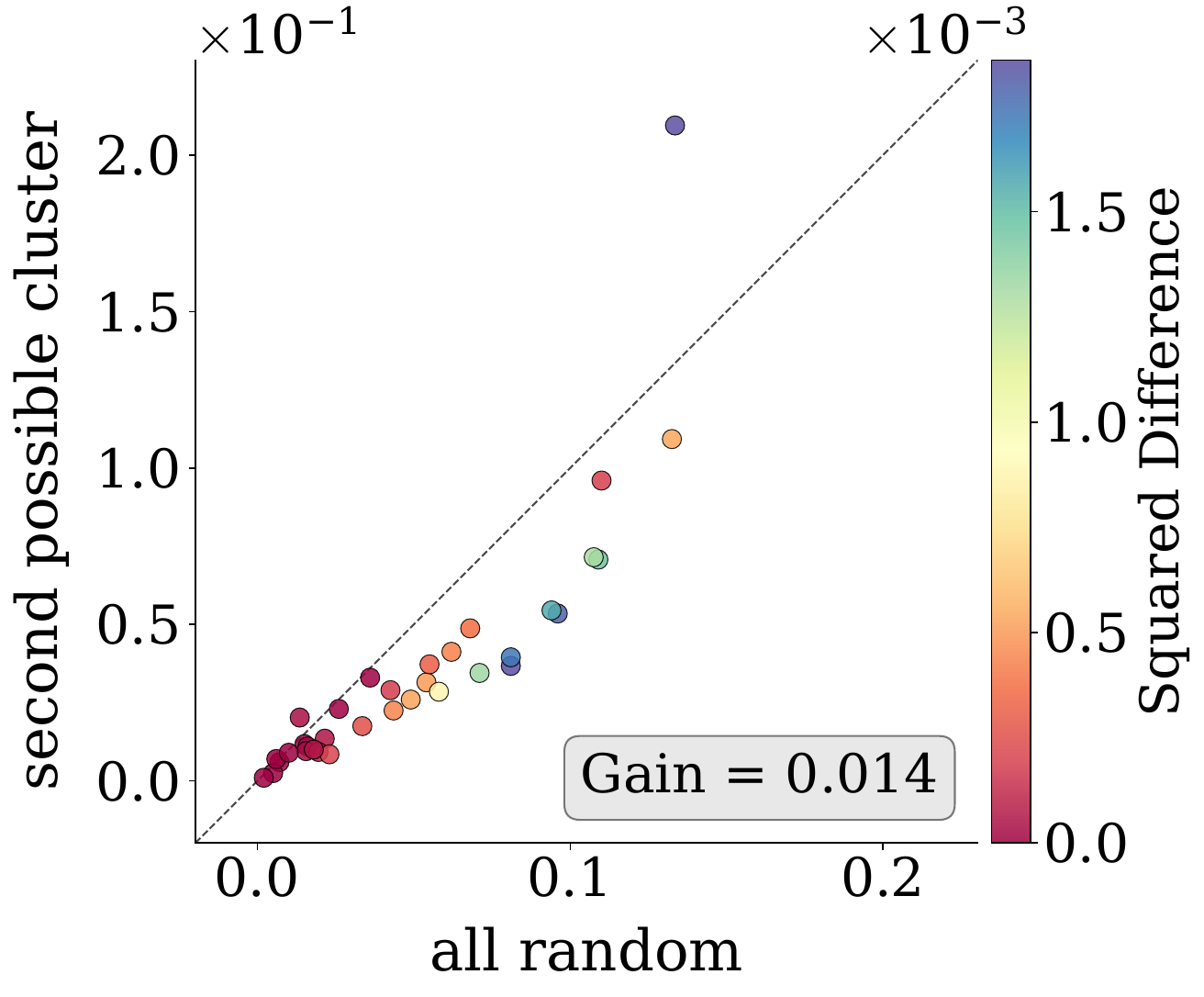}
    \caption{Avg. Flipping Cost}
  \end{subfigure}
  \medskip
  \begin{subfigure}{0.5\columnwidth}
    \includegraphics[height=0.8\linewidth]{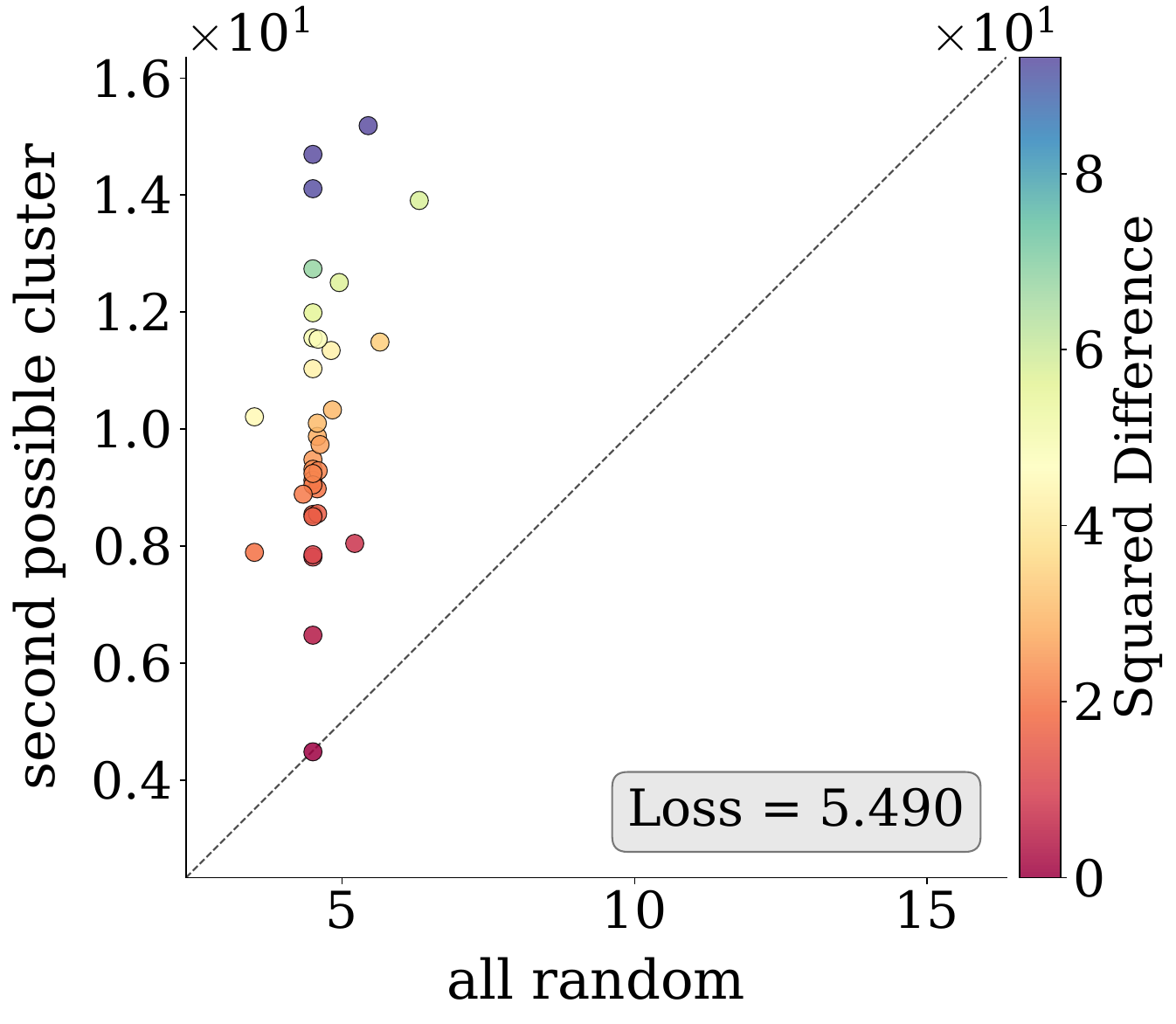}
    \caption{Avg. Changed Segments}
  \end{subfigure}\hfill
  \begin{subfigure}{0.5\columnwidth}
    \includegraphics[height=0.8\linewidth]{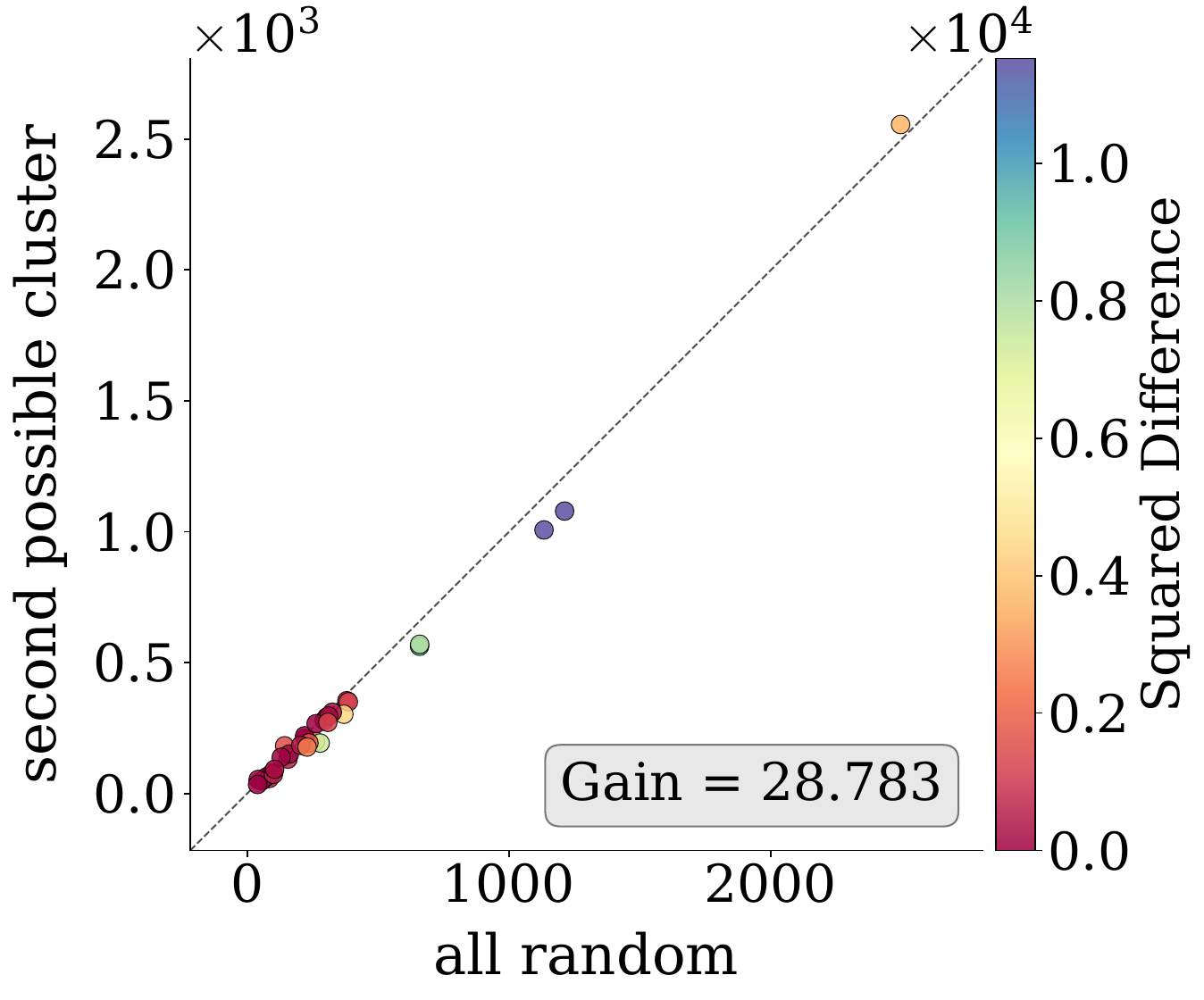}
    \caption{Avg. Changed Timesteps}
  \end{subfigure}
  \caption{\textsc{Galactic-L} Target Selection Analysis. Comparison of the second possible and the all random policies across the UCR datasets.}
  \label{fig:additional_direction_scatter_second_allrandom}
\end{figure}

\begin{figure}
  \centering
  \begin{subfigure}{0.5\columnwidth}
    \includegraphics[height=0.8\linewidth]{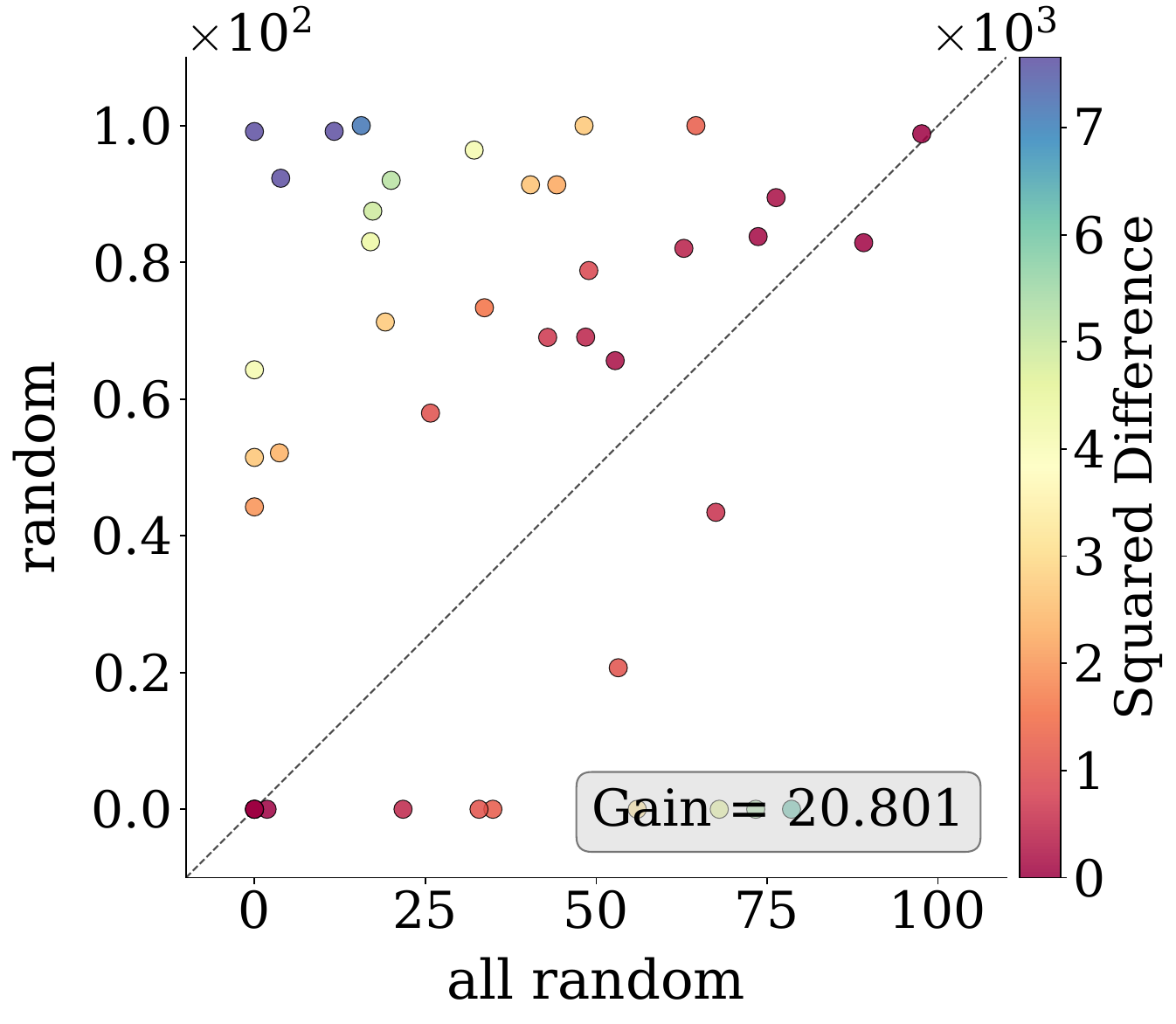}
    \caption{Effectiveness}
  \end{subfigure}\hfill
  \begin{subfigure}{0.5\columnwidth}
    \includegraphics[height=0.8\linewidth]{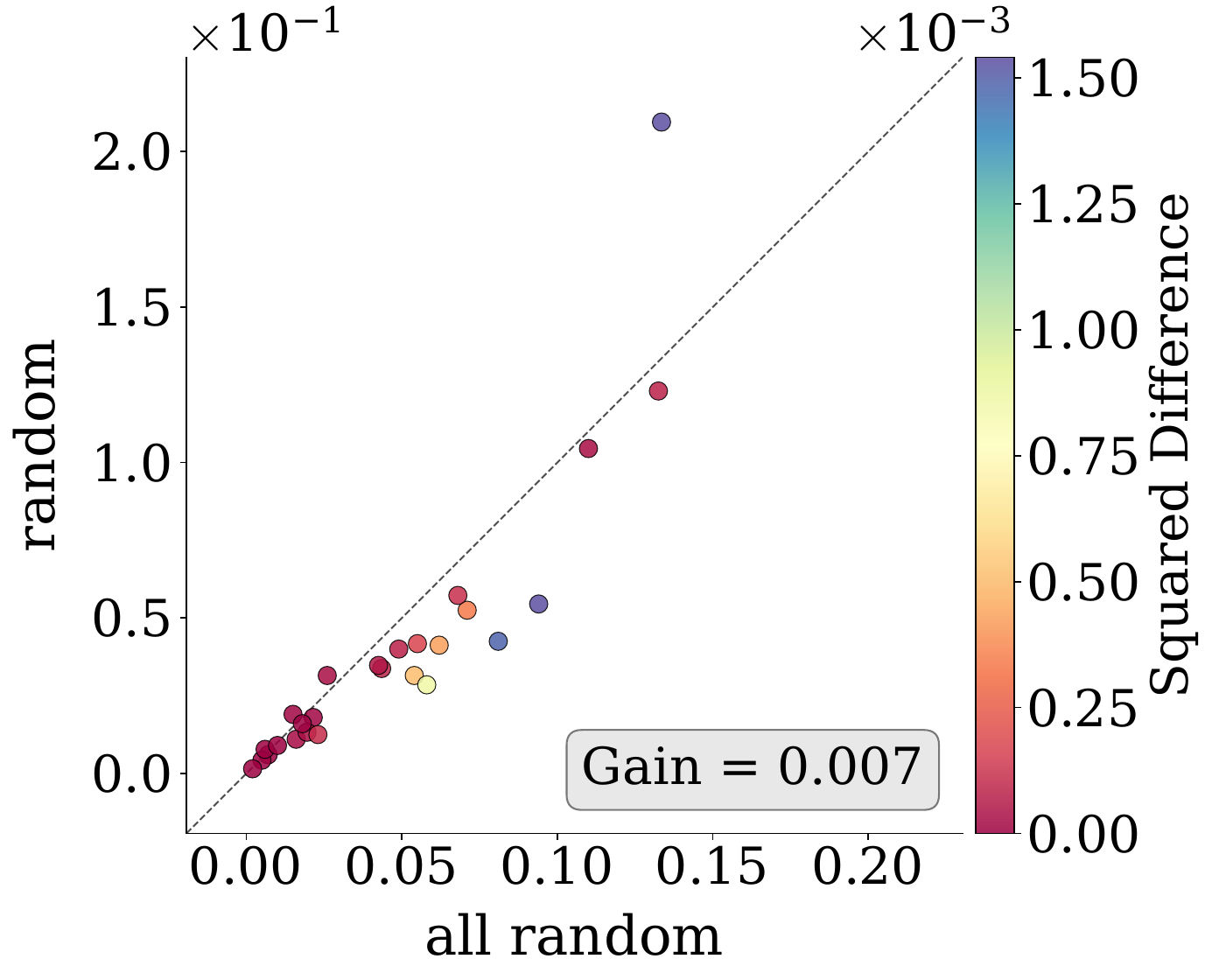}
    \caption{Avg. Flipping Cost}
  \end{subfigure}
  \medskip
  \begin{subfigure}{0.5\columnwidth}
    \includegraphics[height=0.8\linewidth]{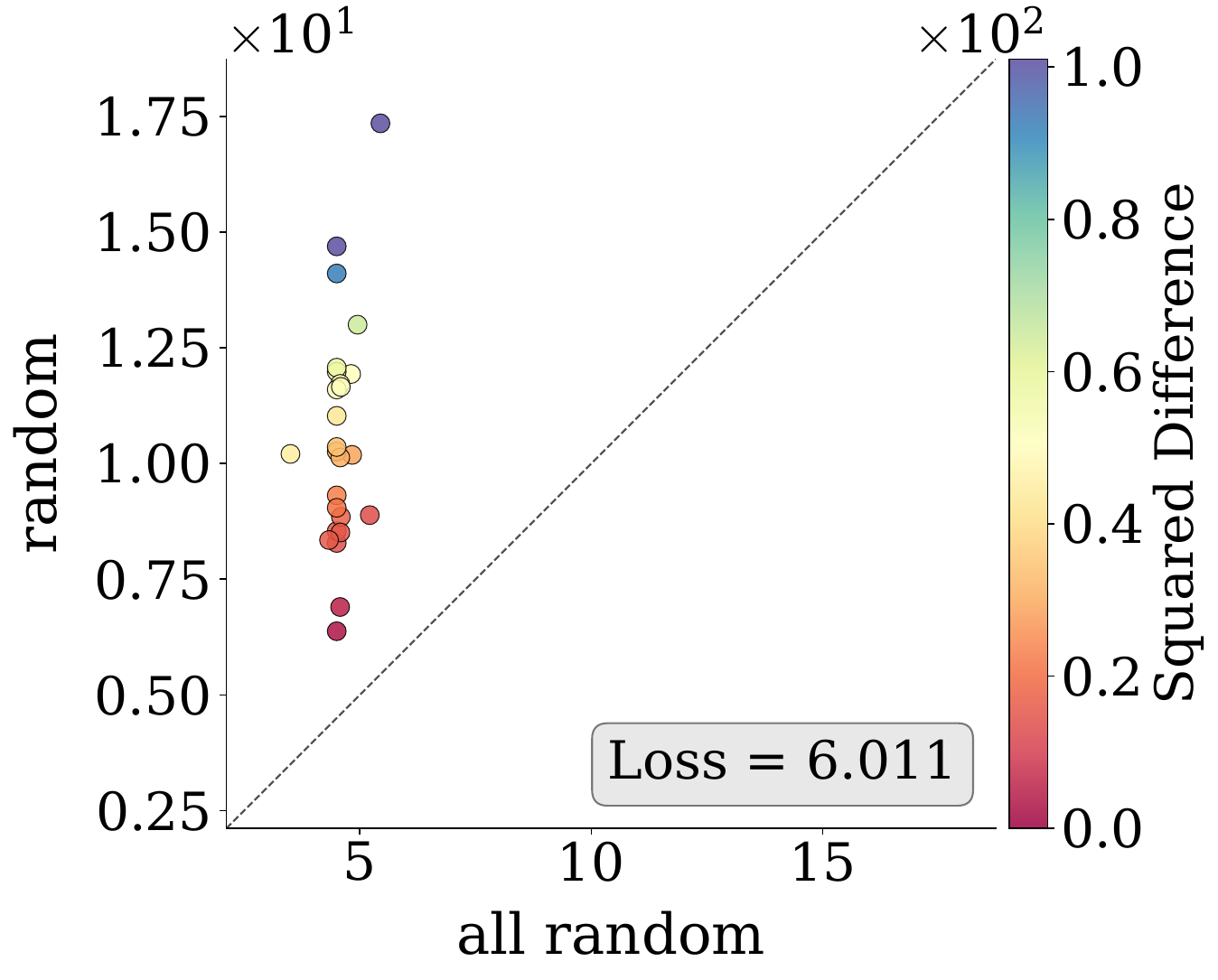}
    \caption{Avg. Changed Segments}
  \end{subfigure}\hfill
  \begin{subfigure}{0.5\columnwidth}
    \includegraphics[height=0.8\linewidth]{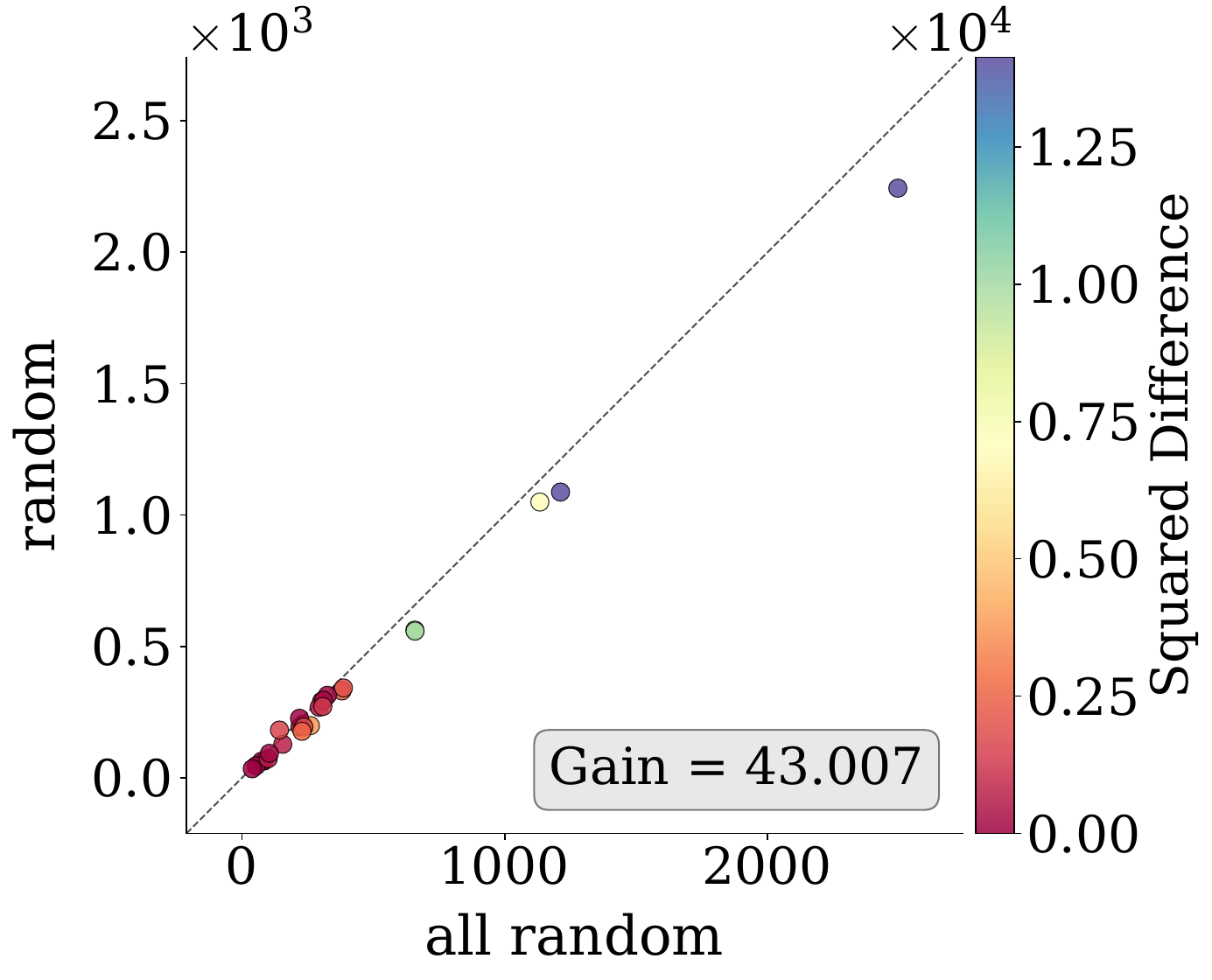}
    \caption{Avg. Changed Timesteps}
  \end{subfigure}
  \caption{\textsc{Galactic-L} Target Selection Analysis. Comparison of the random and the all random policies across the UCR datasets.}
  \label{fig:additional_direction_scatter_random_all_random}
\end{figure}

\subsection{Search Policy: Target Selection}
\label{sec:target_selection_appendix}
We further analyze target specification by comparing \textit{second possible} and \textit{random} policies against \textit{all\_random}, where no explicit target is defined, and optimization is driven by pushing away from the source cluster. \Cref{fig:additional_direction_scatter_second_allrandom} compares \textit{second possible} versus \textit{all\_random}, while \Cref{fig:additional_direction_scatter_random_all_random} compares \textit{random} versus \textit{all\_random} across UCR datasets.

Relative to \textit{all\_random}, the \textit{second possible} policy yields a substantial gain in effectiveness, consistently achieving higher success rates. This gain is accompanied by a loss in segment-level sparsity, reflected in a larger number of modified segments, but a gain at the timestep level, with fewer total modified timesteps. This pattern indicates that directing the search toward a specific alternative cluster may require interacting with more semantic regions, while still keeping the actual temporal changes compact.
A similar but weaker trend is observed for the \textit{random} policy. Explicitly selecting a target cluster at random improves effectiveness compared to \textit{all\_random}, but with smaller gains than \textit{second possible} in effectiveness. 

Overall, these results show that explicitly defining a target cluster, particularly when guided by the probability structure of the surrogate, prioritizes successful counterfactual discovery and concentrates perturbations within short temporal intervals, accepting moderate increases in segment-level edits in exchange for substantially higher effectiveness.

\subsection{\textsc{Galactic-G} Sensitivity Analysis: Effectiveness vs. Efficiency}
\label{sec:galactic_g_varied_m}
Within \textsc{Galactic-G}, we evaluate the trade-off between explanatory quality and efficiency by varying the perturbation budget $\mu$. As shown in \Cref{fig:galactic_g_varied_m}, which illustrates this analysis for the \textbf{ECG200} dataset, increasing $\mu$ yields a non-linear, monotonic improvement in effectiveness. Initially, the \textbf{Minimum Description Length (MDL)} objective prioritizes high-impact, parsimonious perturbations that maximize cluster effectiveness while minimizing aggregate flipping costs, with the minimum description length. For this case, results plateau once $\mu$ exceeds the intrinsic structural limit of three groups, suggesting a singular, well-optimized perturbation per group sufficiently characterizes the cluster manifold, rendering further counterfactuals redundant.

Furthermore, our analysis validates that while the \textbf{hierarchical optimal} achieves local optimality within discrete subgroups, it remains globally sub-optimal when evaluated at the aggregate cluster level. This global sub-optimality arises because the hierarchical decomposition lacks the holistic search visibility required to evaluate cross-group interactions and unified description costs across the entire cluster partition. Regarding efficiency, the \textbf{greedy selection} mechanism offers a computationally superior alternative to the optimal search, delivering comparable effectiveness with significantly reduced runtime. While the \textbf{hierarchical greedy} variant is the most efficient for small budgets, the computational overhead associated with propagating search operations across numerous subgroups causes its complexity to converge with the standard greedy approach as $\mu$ increases. These findings confirm that the \textsc{Galactic} effectively balances the combinatorial challenge of global selection with actionable, high-fidelity explanations.

\begin{figure}
  \label{fig:galactic_g_varied_m}
  \centering
  \begin{subfigure}{0.48\columnwidth}
    \centering
    \includegraphics[width=0.9\linewidth]{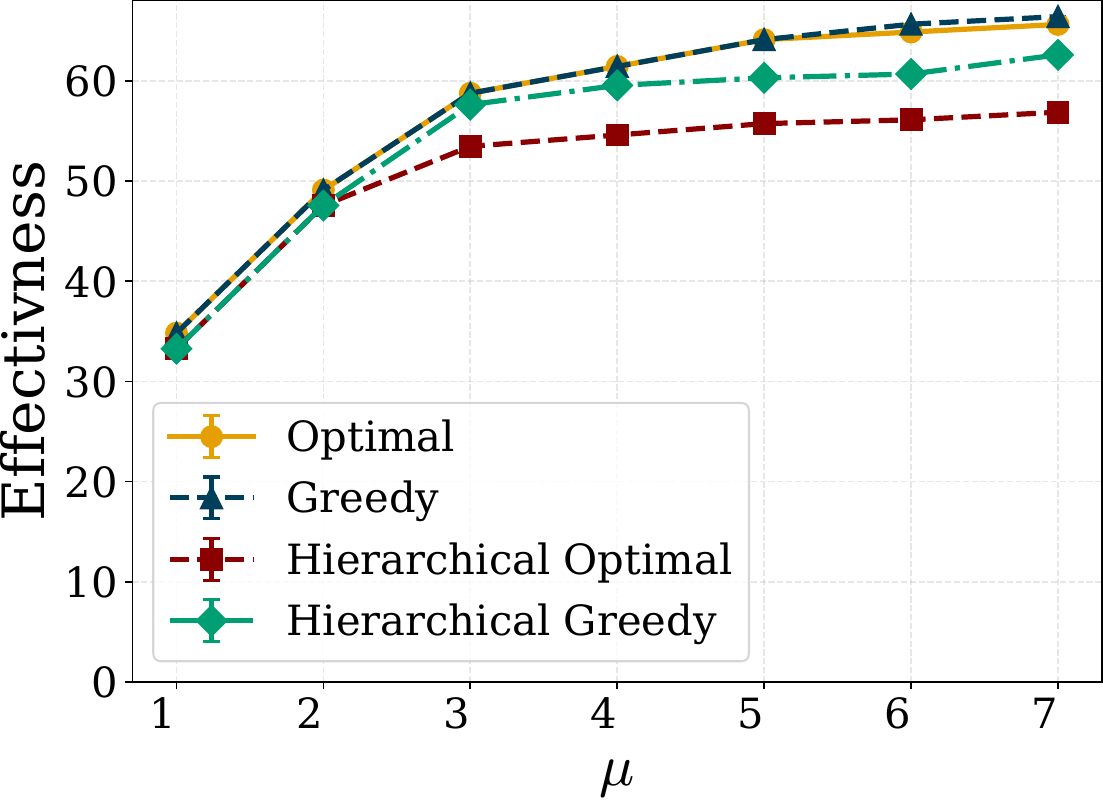}
    \caption{Effectiveness}
  \end{subfigure}\hfill
  \begin{subfigure}{0.48\columnwidth}
    \centering
    \includegraphics[width=0.9\linewidth]{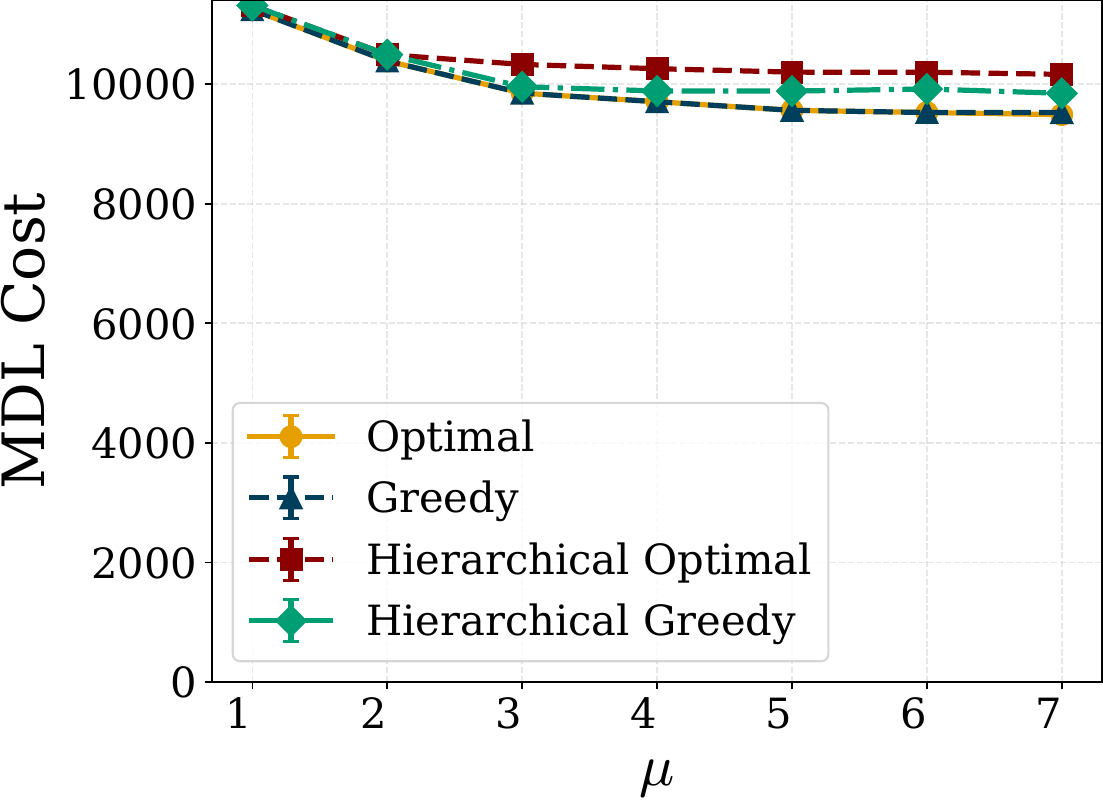}
    \caption{MDL Cost}
  \end{subfigure}  
  \begin{subfigure}{0.48\columnwidth}
    \centering
    \includegraphics[width=0.9\linewidth]{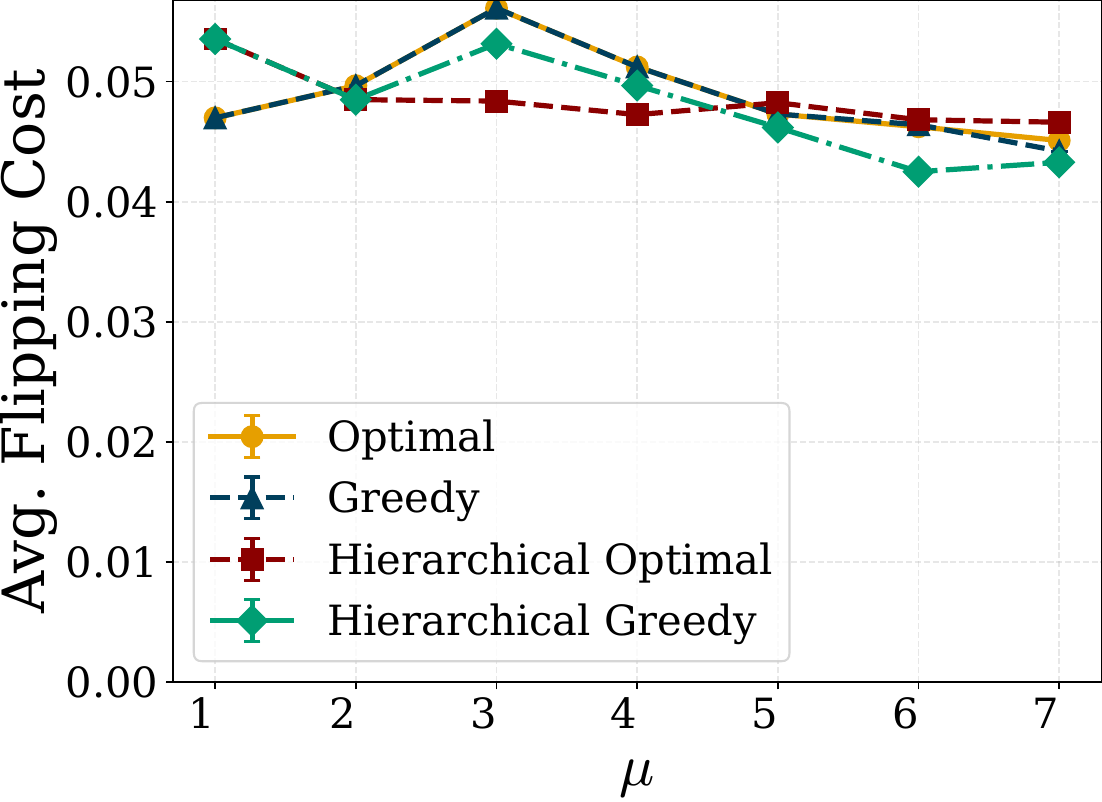}
    \caption{Average Flipping Cost}
  \end{subfigure}\hfill
  \begin{subfigure}{0.48\columnwidth}
    \centering
    \includegraphics[width=0.9\linewidth]{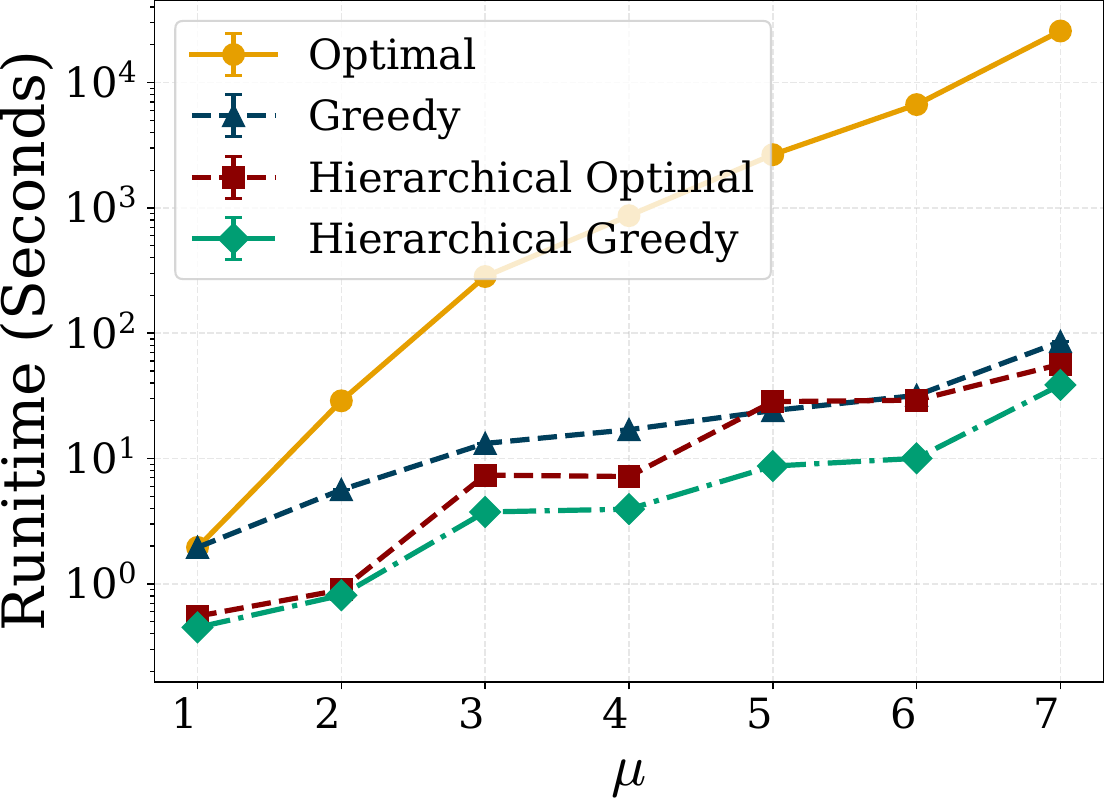}
    \caption{Runtime}
  \end{subfigure}
  \caption{\textsc{Galactic-G} selections with varied $\mu$}
\end{figure}

\begin{table*}
\centering
\caption{Metadata of the selected UCR time-series datasets and Surrogate classification accuracies.}
\label{tab:dataset_summary}
\resizebox{\textwidth}{!}{%
\begin{tabular}{lrrrlr | lrrrlr}
\toprule
\textbf{Dataset} & \textbf{Instances} & \textbf{Length} & \textbf{Clusters} & \textbf{Type} & \textbf{ResNet Acc.} & 
\textbf{Dataset} & \textbf{Instances} & \textbf{Length} & \textbf{Clusters} & \textbf{Type} & \textbf{ResNet Acc.} \\
\midrule
ArrowHead & 211 & 251 & 3 & IMAGE & 73.81\% & ItalyPowerDemand & 1096 & 24 & 2 & SENSOR & 95.45\% \\
Lightning7 & 143 & 319 & 7 & SENSOR & 75.86\% & CBF & 930 & 128 & 3 & SIMUL. & 100.00\% \\
Car & 120 & 577 & 4 & SENSOR & 91.67\% & MiddlePhalanxTW & 553 & 80 & 6 & IMAGE & 58.56\% \\
Coffee & 56 & 286 & 2 & SPECTRO & 100.00\% & Computers & 500 & 720 & 2 & DEVICE & 92.00\%  \\
CricketZ & 780 & 300 & 12 & HAR & 84.62\% & PowerCons & 360 & 144 & 2 & DEVICE & 82.61\% \\
Earthquakes & 461 & 512 & 2 & SENSOR & 75.27\% & Rock & 70 & 2844 & 4 & SPECTRO & 78.57\% \\
ECG5000 & 5000 & 140 & 5 & ECG & 94.50\% & ShapeletSim & 200 & 500 & 2 & SIMUL. & 100.00\% \\
ECGFiveDays & 884 & 136 & 2 & ECG & 100.00\% & ToeSegmentation2 & 166 & 343 & 2 & SENSOR & 97.06\% \\
FaceFour & 112 & 350 & 4 & IMAGE & 100.00\% & TwoLeadECG & 1162 & 82 & 2 & ECG & 100.00\% \\
FordB & 4446 & 500 & 2 & SENSOR & 90.00\% & Wine & 111 & 234 & 2 & SPECTRO & 47.83\% \\
Ham & 214 & 431 & 2 & SPECTRO & 86.05\% & Yoga & 3300 & 426 & 2 & IMAGE & 96.67\% \\
InlineSkate & 650 & 1882 & 7 & MOTION & 52.85\% & Wafer & 7164 & 152 & 2 & SENSOR & 99.93\% \\
ACSF1 & 200 & 1460 & 10 & DEVICE & 80.00\% & ECG200 & 200 & 96 & 2 & ECG & 90.00\% \\
Fish & 350 & 463 & 7 & IMAGE & 95.71\% & SyntheticControl & 600 & 60 & 6 & SIMUL. & 100.00\% \\
UMD & 180 & 150 & 3 & SIMUL. & 100.00\% & Worms & 258 & 900 & 5 & MOTION & 62.22\% \\
\bottomrule
\end{tabular}}
\end{table*}

\begin{table*}[t]
\centering
\caption{
Additional Local Performance Benchmarking. Comparative analysis of \textsc{Galactic-L} against \textsc{kNN}, \textsc{TSEvo}, and \textsc{Glacier} variants across effectiveness, average flipping cost, average segments changed, average timesteps changed, and runtime. The null indicator ``-'' is assigned throughout the metrics where the corresponding effectiveness is zero.
}
\label{tab:additional_local_results}
\resizebox{\textwidth}{!}{%
\tiny
\setlength{\tabcolsep}{4pt}
\begin{tabular}{@{} l ccccc  ccccc  ccccc  ccccc  ccccc @{}}
\toprule
\multirow{2}{*}{Dataset} & \multicolumn{5}{c}{\textsc{KNN}} & \multicolumn{5}{c}{\textsc{TSEvo}} & \multicolumn{5}{c}{\textsc{Glacier-L}} & \multicolumn{5}{c}{\textsc{Glacier-G}} & \multicolumn{5}{c}{\textsc{Galactic-L}} \\
\cmidrule(lr){2-6} \cmidrule(lr){7-11} \cmidrule(lr){12-16} \cmidrule(lr){17-21} \cmidrule(lr){22-26}
 & eff& afc & asc & atc & RT &  eff& afc & asc & atc & RT & eff& afc & asc & atc & RT & eff& afc & asc & atc & RT &  eff& afc & asc & atc & RT \\
\midrule
 Computers & 100 & 0.1 & 29.83 & 720 & 0.41 & 100 & 0.07 & 29.83 & 720 & 906.87 & 35.33 & 0 & 29.68 & 640.91 & 31.52 & 36 & 0 & 29.72 & 648.85 & 34.47 & 87.33 & 0.01 & 8.53 & 399.92 & 29.42\\
 CricketZ & 100 & 0.09 & 30.22 & 300 & 1.01 & 100 & 0.1 & 30.22 & 300 & 2366.73 & 11.97 & 0.01 & 32.11 & 300 & 22.94 & 14.53 & 0.01 & 32.97 & 300 & 26.57 & 99.57 & 0.01 & 8 & 190.56 & 75.75\\
 Earthquakes & 100 & 0.24 & 39.67 & 512 & 0.37 & 100 & 0.19 & 39.67 & 512 & 1289.46 & 31.65 & 0.01 & 41.41 & 509.32 & 27.06 & 34.53 & 0 & 41.56 & 509.77 & 29.7 & 94.24 & 0.03 & 10.97 & 294.5 & 65.15\\
 ECG5000 & 100 & 0.07 & 20.15 & 140 & 7.12 & 100 & 0.14 & 20.15 & 140 & 9583.36 & 3.27 & 0.01 & 26.51 & 133.78 & 28.94 & 5.13 & 0.01 & 24.45 & 128.92 & 44.97 & 61.8 & 0.01 & 7.43 & 50.2 & 282.68\\
 ECGFiveDays & 100 & 0.11 & 22.79 & 136 & 1.47 & 100 & 0.09 & 22.79 & 136 & 1737.3 & 22.93 & 0.01 & 22.44 & 132.8 & 36.41 & 52.63 & 0.02 & 22.51 & 129.69 & 78.49 & 100 & 0.03 & 5.94 & 36.84 & 121.5\\
 FaceFour & 100 & 0.2 & 29.5 & 350 & 0.19 & 100 & 0.21 & 29.5 & 350 & 339.1 & 17.65 & 0.02 & 29.33 & 350 & 4.67 & 5.88 & 0.01 & 31 & 350 & 1.46 & 100 & 0.04 & 8.91 & 242.62 & 16.06\\
 FordB & 100 & 0.14 & 30.1 & 500 & 16.04 & 99.93 & 0.08 & 30.11 & 500 & 23812.4 & 23.31 & 0 & 44.1 & 498.73 & 287.05 & 27.66 & 0 & 44.07 & 498.78 & 353.71 & 62.44 & 0.01 & 11.32 & 244.14 & 104.02\\
 GesturePebbleZ2 & 100 & 0.09 & 28.43 & 451.2 & 0.41 & 100 & 0.16 & 28.27 & 450.32 & 545.86 & 0 & - & - & - & 0 & - & - & - & - & 0 & 100 & 0.01 & 7.41 & 321.45 & 141.58\\
 GunPoint & 100 & 0.14 & 27.53 & 150 & 0.27 & 100 & 0.1 & 27.53 & 150 & 333.07 & 45 & 0.02 & 28.85 & 149.96 & 17.55 & 48.33 & 0.02 & 28.83 & 150 & 30.31 & 96.67 & 0.06 & 7.53 & 62.55 & 13.07\\
 Ham & 100 & 0.16 & 33.98 & 431 & 0.09 & 100 & 0.14 & 33.98 & 431 & 379.56 & 50.77 & 0.01 & 33.85 & 420.39 & 22.04 & 61.54 & 0.01 & 34.88 & 416.38 & 38.64 & 98.46 & 0.03 & 8.97 & 280.91 & 8.61\\
 InlineSkate & 100 & 0.06 & 23.57 & 1882 & 0.94 & 100 & 0.1 & 23.57 & 1882 & 1469.5 & 68.11 & 0 & 23.5 & 1566.22 & 56.12 & 65.41 & 0 & 23.43 & 1552.45 & 103.88 & 98.92 & 0.01 & 6.51 & 1146.64 & 34.79\\
ItalyPowerDemand & 0 & - & - & - & 0.08 & 100 & 0.1 & 6.89 & 24 & 2821.62 & 8.51 & 0.02 & 8.32 & 24 & 15.33 & 15.81 & 0.01 & 8.38 & 24 & 29.17 & 73.86 & 0.05 & 2.17 & 15.49 & 71.84 \\
 Lightning7 & 100 & 0.14 & 26.58 & 319 & 0.09 & 100 & 0.15 & 26.58 & 319 & 141.8 & 0 & - & - & - & - & 6.98 & 0.01 & 27 & 319 & 3.2 & 95.35 & 0.04 & 7.32 & 169.46 & 27.52\\
 MiddlePhalanxTW & 100 & 0.07 & 17.93 & 80 & 0.25 & 100 & 0.14 & 17.93 & 80 & 533.81 & 60.24 & 0.01 & 18.35 & 79.99 & 56.28 & 69.88 & 0.01 & 18.22 & 79.99 & 90.65 & 100 & 0.05 & 4.86 & 43.48 & 26.34\\
 PowerCons & 100 & 0.21 & 38.68 & 141.57 & 0.09 & 100 & 0.15 & 38.71 & 142.94 & 213.43 & 13.04 & 0.01 & 43.22 & 144 & 6 & 8.7 & 0.01 & 43.67 & 144 & 6.35 & 86.96 & 0.04 & 11.33 & 58.82 & 18\\
 ShapeletSim & 100 & 0.38 & 26.52 & 500 & 0.08 & 100 & 0.28 & 26.52 & 500 & 194.66 & 0 & - & - & - & - & 0 & - & - & - & 0 & 63.33 & 0.11 & 9.82 & 264.89 & 29.9\\
ToeSegmentation2 & 100 & 0.2 & 26.94 & 343 & 100 & 0.17 & 26.94 & 343 & 14 & 0.01 & 29.43 & 343 & 10 & 0.01 & 29 & 343 & 100 & 0.06 & 7.06 & 129.7 & 0.07 & 160.02 & 3.29 & 3.37 & 51.48 \\
 TwoLeadECG & 100 & 0.08 & 22.61 & 82 & 0.47 & 100 & 0.08 & 22.61 & 82 & 1084.46 & 12.89 & 0.02 & 26.2 & 82 & 28.89 & 49.57 & 0.02 & 24.01 & 81.99 & 166.6 & 91.98 & 0.04 & 6.49 & 18.69 & 76.35\\
Wafer & 0 & - & - & - & 0.21 & 100 & 0.02 & 18.16 & 151.99 & 16581.7 & 0 & - & - & - & 0.11 & 0 & - & - & - & 0.09 & 81.44 & 0.01 & 5.43 & 32.87 & 257.01 \\
Yoga & 0 & - & - & - & 0.33 & 100 & 0.04 & 24.77 & 426 & 9700.39 & 48.08 & 0 & 24.45 & 423.7 & 273.55 & 0 & - & - & - & 11.02 & 97.88 & 0.01 & 6.64 & 219.55 & 223.98 \\ 
\bottomrule
\end{tabular}
}
\end{table*}

\begin{table*}[t]
\centering
\caption{Additional Global Performance Benchmarking. Comparative analysis of \textsc{Galactic-G} variants against \textsc{GLOBE-CE}$^*$ and \textsc{Glacier-G}$^*$ across effectiveness, average flipping cost, average segments changed, average timesteps changed, and runtime. The $^*$ indicates adaptations for the global timeseries clustering context; the null indicator ``-'' is assigned throughout the metrics where the corresponding effectiveness is zero.}
\label{tab:additional_global_results}
\resizebox{\textwidth}{!}{%
\begin{tabular}{lcccccccccccccccccccccccccccccc}
\toprule
& \multicolumn{20}{c}{\textsc{Galactic-G}} & \multicolumn{10}{c}{\textsc{Competitors}} \\
\cmidrule(lr){2-21} \cmidrule(lr){22-31}
 Dataset & \multicolumn{5}{c}{\textsc{Optimal}} & \multicolumn{5}{c}{\textsc{Greedy}} & \multicolumn{5}{c}{\textsc{Hierarchical}} & \multicolumn{5}{c}{\textsc{Hierarchical Greedy}} & \multicolumn{5}{c}{\textsc{GLOBE-CE$^*$}} & \multicolumn{5}{c}{\textsc{Glacier-G$^*$}}\\
\cmidrule(lr){2-6} \cmidrule(lr){7-11} \cmidrule(lr){12-16} \cmidrule(lr){17-21} \cmidrule(lr){22-26} \cmidrule(lr){27-31}
  & eff& afc & asc & atc & RT &  eff& afc & asc & atc & RT & eff& afc & asc & atc & RT & eff& afc & asc & atc & RT &  eff& afc & asc & atc & RT &  eff& afc & asc & atc & RT \\
\midrule
Computers & 36.9 & 0.006 & 5.0 & 580 & 303.33 & 36.9 & 0.006 & 5.0 & 580 & 25.60 & 34.7 & 0.007 & 5.0 & 577 & 17.03 & 34.9 & 0.006 & 5.0 & 579 & 8.12 & 42.1 & 0.012 & 7.8 & 720 & 218.32 & 8.2 & 0.003 & 12.5 & 720 & 15.24 \\
CricketZ & 29.4 & 0.014 & 3.6 & 178 & 145.44 & 29.0 & 0.012 & 3.8 & 175 & 10.90 & 15.6 & 0.009 & 4.1 & 167 & 6.35 & 15.7 & 0.009 & 4.1 & 168 & 5.09 & 85.8 & 0.032 & 6.4 & 275 & 29.49 & 4.4 & 0.006 & 5.8 & 250 & 25.44 \\
Earthquakes & 50.1 & 0.055 & 3.7 & 241 & 303.97 & 50.1 & 0.055 & 3.7 & 241 & 30.57 & 36.1 & 0.029 & 4.0 & 233 & 20.74 & 40.3 & 0.034 & 3.9 & 234 & 15.52 & 21.3 & 0.034 & 7.0 & 512 & 139.45 & 3.1 & 0.005 & 7.0 & 512 & 20.94 \\
ECG5000 & 48.3 & 0.018 & 4.6 & 37 & 778.67 & 48.1 & 0.019 & 4.6 & 37 & 41.15 & 44.7 & 0.018 & 4.6 & 36 & 47.73 & 45.2 & 0.017 & 4.5 & 35 & 31.83 & 75.7 & 0.106 & 8.4 & 140 & 175.61 & 0.2 & 0.000 & 3.0 & 27 & 0.00 \\
ECGFiveDays & 82.6 & 0.038 & 3.1 & 38 & 5367.02 & 82.6 & 0.038 & 3.1 & 38 & 65.79 & 66.5 & 0.034 & 3.1 & 37 & 50.92 & 66.4 & 0.036 & 3.1 & 39 & 20.98 & 50.0 & 0.138 & 7.2 & 136 & 91.62 & 21.3 & 0.019 & 7.6 & 136 & 0.00 \\
FaceFour & 50.6 & 0.037 & 3.8 & 178 & 271.77 & 50.6 & 0.032 & 3.5 & 174 & 6.68 & 25.0 & 0.032 & 3.1 & 169 & 6.56 & 25.0 & 0.030 & 3.0 & 168 & 4.10 & 55.3 & 0.059 & 5.2 & 262 & 40.68 & 0.0 & - & - & - & 20.83 \\
FordB & 10.7 & 0.013 & 3.4 & 142 & 25204.26 & 10.7 & 0.013 & 3.4 & 142 & 590.80 & 1.8 & 0.007 & 2.6 & 149 & 324.76 & 1.8 & 0.006 & 2.6 & 149 & 143.58 & 30.5 & 0.021 & 7.0 & 500 & 1596.26 & 1.0 & 0.003 & 7.0 & 499 & 18.00 \\
GesturePebbleZ2 & 53.5 & 0.011 & 4.4 & 296 & 37.23 & 53.5 & 0.011 & 4.4 & 296 & 25.07 & 42.2 & 0.008 & 4.7 & 283 & 22.59 & 42.5 & 0.007 & 4.7 & 285 & 22.58 & 15.0 & 0.002 & 6.0 & 303 & 33.40 & 0.0 & - & - & - & 18.12 \\
GunPoint & 76.5 & 0.109 & 3.3 & 56 & 119.67 & 76.5 & 0.094 & 3.3 & 56 & 13.30 & 65.5 & 0.135 & 3.7 & 60 & 13.20 & 65.5 & 0.078 & 3.6 & 60 & 7.53 & 57.0 & 0.277 & 7.0 & 150 & 36.44 & 18.5 & 0.023 & 7.0 & 150 & 12.39 \\
Ham & 51.8 & 0.034 & 4.7 & 259 & 226.36 & 51.8 & 0.034 & 4.7 & 259 & 15.22 & 50.3 & 0.033 & 4.5 & 245 & 18.96 & 50.3 & 0.033 & 4.5 & 245 & 7.26 & 5.7 & 0.069 & 7.5 & 431 & 61.08 & 8.1 & 0.010 & 12.1 & 431 & 19.55 \\
InlineSkate & 89.1 & 0.007 & 4.4 & 657 & 481.00 & 89.1 & 0.006 & 4.1 & 614 & 15.33 & 86.8 & 0.007 & 4.6 & 814 & 5.57 & 86.8 & 0.006 & 3.9 & 628 & 4.40 & 85.8 & 0.008 & 7.0 & 1882 & 218.86 & 27.7 & 0.001 & 6.0 & 1603 & 20.08 \\
ItalyPowerDemand & 49.8 & 0.064 & 2.3 & 12 & 259.87 & 49.8 & 0.059 & 2.3 & 11 & 26.30 & 48.8 & 0.076 & 2.2 & 11 & 1.75 & 49.6 & 0.057 & 2.4 & 12 & 6.87 & 70.6 & 0.097 & 3.2 & 24 & 55.17 & 2.8 & 0.012 & 1.7 & 12 & 0.00 \\
Lightning7 & 60.5 & 0.055 & 4.0 & 145 & 36.24 & 60.5 & 0.050 & 3.9 & 140 & 6.85 & 25.8 & 0.021 & 3.3 & 110 & 6.00 & 26.6 & 0.018 & 3.1 & 102 & 5.86 & 36.7 & 0.033 & 4.0 & 182 & 33.41 & 8.0 & 0.006 & 4.0 & 182 & 9.35 \\
MiddlePhalanxTW & 86.5 & 0.046 & 3.5 & 43 & 7.40 & 86.5 & 0.045 & 3.3 & 42 & 3.55 & 68.3 & 0.042 & 3.4 & 41 & 1.73 & 68.3 & 0.042 & 3.4 & 41 & 1.72 & 100.0 & 0.133 & 7.0 & 80 & 27.94 & 41.1 & 0.014 & 7.0 & 80 & 10.71 \\
PowerCons & 37.7 & 0.048 & 6.1 & 57 & 1548.78 & 37.7 & 0.048 & 6.1 & 57 & 29.43 & 30.5 & 0.040 & 6.0 & 53 & 49.89 & 37.7 & 0.048 & 6.1 & 57 & 26.99 & 47.2 & 0.141 & 10.3 & 144 & 46.46 & 2.3 & 0.011 & 24.7 & 144 & 15.96 \\
ShapeletSim & 45.5 & 0.135 & 5.5 & 247 & 110.65 & 45.5 & 0.112 & 5.5 & 245 & 14.05 & 45.5 & 0.135 & 5.5 & 247 & 11.68 & 45.5 & 0.112 & 5.5 & 247 & 9.70 & 0.0 & - & - & - & 74.31 & 0.0 & - & - & - & 18.25 \\
ToeSegmentation2 & 67.2 & 0.084 & 4.9 & 142 & 173.25 & 67.2 & 0.084 & 4.9 & 142 & 22.74 & 61.6 & 0.060 & 5.1 & 143 & 22.71 & 61.6 & 0.060 & 5.1 & 143 & 17.82 & 22.7 & 0.033 & 7.8 & 343 & 46.68 & 3.6 & 0.013 & 16.9 & 343 & 16.57 \\
TwoLeadECG & 39.7 & 0.035 & 3.3 & 28 & 1683.43 & 39.7 & 0.034 & 3.3 & 28 & 60.38 & 13.8 & 0.032 & 3.3 & 28 & 66.77 & 33.7 & 0.034 & 3.4 & 28 & 34.42 & 62.7 & 0.127 & 7.0 & 82 & 71.97 & 20.6 & 0.015 & 7.0 & 82 & 9.27 \\
Wafer & 38.3 & 0.018 & 1.3 & 10 & 1158.02 & 38.3 & 0.018 & 1.3 & 10 & 166.75 & 28.1 & 0.011 & 1.3 & 10 & 131.90 & 28.1 & 0.011 & 1.3 & 10 & 72.22 & 50.0 & 0.006 & 3.5 & 76 & 672.45 & 3.9 & 0.002 & 7.0 & 152 & 0.00 \\
Yoga & 48.6 & 0.007 & 5.9 & 230 & 38593.63 & 48.6 & 0.007 & 5.9 & 230 & 416.69 & 45.5 & 0.007 & 5.9 & 217 & 1001.63 & 47.5 & 0.007 & 5.9 & 218 & 225.89 & 57.7 & 0.036 & 7.3 & 426 & 1043.10 & 5.9 & 0.002 & 8.5 & 426 & 17.18 \\
\bottomrule
\end{tabular}
}
\end{table*}

\subsection{Additional Results for Local Evaluation}
\Cref{tab:additional_local_results} presents additional results for 20 UCR datasets.

\subsection{Additional Results for Global Evaluation}
\Cref{tab:additional_global_results} presents additional results for 20 UCR datasets. 

%% file: sections/proofs.tex
\begin{proposition}[Supermodularity of the MDL objective]
\label{proofs:mdl-supermodular}
For a source cluster $C_k$ and the finite set of candidate perturbations $\Delta_k$ for $C_k$, the set function $L(C_k, \mathbb S) = 
  L(\mathbb S) + L(C_k \mid \mathbb S)$ is \emph{supermodular} on $\Delta_k$.
That is, for all $\A \subseteq \B \subseteq \Delta_k$ and all $\vec{\delta} \in \Delta_k \setminus \B$, under \Cref{ass:admissible}:
\begin{equation*}
  L(D, \A \cup \{\vec{\delta}\}) - L(D, \A)
  \;\le\;
  L(D, \B \cup \{\vec{\delta}\}) - L(D, \B).
  % \label{eq:supermodularity_def}
\end{equation*}
\end{proposition}

\begin{proof}
We analyze the marginal change in the total description length upon adding a new perturbation $\vec{\delta} \notin \mathbb S$ to an existing summary set $\mathbb S$.
We define the uncovered set 
\(
\uncov{\mathbb S} \;=\; C_k \setminus \mathrm{cov}(C_k, \mathbb S).
\).
% Let the $p\_sz$ denote the bit-cost of a pointer.
The reference perturbation is the maximal-description perturbation
\[
\vec{\delta}^*
=
\arg\max_{\vec{\delta} \in \Delta_k}
\Bigl(
\mathrm{bits}(\|\vec{\delta}\|_0)
+
\log_2(\|\vec{\delta}\|_1)
+
2p\_sz
\Bigr),
\]
fixed and independent of $\mathbb S$.
We will refer to the cost of any uncovered series as $MC$, where:
\[
MC
=
\mathrm{bits}(\|\vec{\delta}^*\|_0)
+
\log_2(\|\vec{\delta}^*\|_1)
+
2p\_sz.
\]
Then $MC>0$ and in particular $MC \ge 2p\_sz > p\_sz$.

When analyzing the addition of a perturbation to our model
$\vec{\delta} \notin \mathbb S$,
we only allow admissible additions that newly cover at least one
previously uncovered instance, i.e.,
\(
\mathrm{cov}(C_k,\mathbb S\cup\{\vec{\delta}\})
\setminus
\mathrm{cov}(C_k,\mathbb S)
\neq \emptyset.
\)
% $\uncov(\{\vec \delta\}) \neq \emptyset$
The marginal coverage gain provided by adding $\vec{\delta}$ to $\mathbb S$ is the magnitude of the set of instances covered by $\vec{\delta}$ that were not previously covered by $\mathbb S$:
\begin{equation}
\begin{split}
\gamma(\{\vec{\delta}\} \mid \mathbb S)
&= |\mathrm{cov}(C_k, \mathbb S \cup \{\vec{\delta}\}) \setminus \mathrm{cov}(C_k, \mathbb S)|
\\&= |\mathrm{cov}(C_k, \{\vec{\delta}\}) \cap \uncov{\mathbb S}|\\
% \gamma(\{\vec{\delta}\}\mid \mathbb S) &
&= |\mathrm{cov}(C_k,\mathbb S\cup\{\vec{\delta}\})|-|\mathrm{cov}(C_k,\mathbb S)|
\end{split}
\label{eq:gain}
\end{equation}
% Let $\gamma(\{\vec{\delta}\} \mid \mathbb S) = |\mathrm{Gain}(\{\vec{\delta}\}\mid \mathbb S)|$ be the magnitude of this gain.
Since $\A \subseteq \B$ implies $\mathrm{cov}(C_k, \A) \subseteq \mathrm{cov}(C_k, \B)$, we have $\uncov{\B} \subseteq \uncov{\A}$ and thus
\begin{multline}
(\mathrm{cov}(C_k, \{\vec{\delta}\}) \cap \uncov{\B)} \subseteq (\mathrm{cov}(C_k, \{\vec{\delta}\}) \cap \uncov{\A}) \\
\implies \gamma(\{\vec{\delta}\} \mid \B) \le \gamma(\{\vec{\delta}\} \mid \A).
\label{eq:diminishing_coverage}
\end{multline}

The marginal gain of the description length $L(C_k, \mathbb S)$ can be decomposed into the model part and data part.

% -----------------------------
\noindent \textbf{Model part.}
The marginal increment of the model part is:
\begin{align}
\Delta L(\{\vec{\delta}\}\mid \mathbb S)
=
L\bigl(\mathbb S\cup\{\vec{\delta}\})\bigr)-L\bigl(\mathbb S\bigr), \notag
% \label{eq:model_marginal_def}
\end{align}
where
\begin{align}
\begin{split}
L\bigl(\mathbb S\cup\{\vec{\delta}\}\bigr) &= \sum_{\vec{\delta}' \in \mathbb S\cup\{\vec{\delta}\}} \Bigl( \mathrm{bits}(\|\vec{\delta}'\|_0) + \log_2(\|\vec{\delta}'\|_1) \Bigr) \\
&\quad + \bigl(|\mathrm{cov}(C_k,\mathbb S\cup\{\vec{\delta}\})|+|\mathbb S\cup\{\vec{\delta}\}|\bigr)\,p\_sz,
\end{split} \label{eq:model_expand_1} \\
\begin{split}
L\bigl(\mathbb S)\bigr) &= \sum_{\vec{\delta}' \in \mathbb S} \Bigl( \mathrm{bits}(\|\vec{\delta}'\|_0) + \log_2(\|\vec{\delta}'\|_1) \Bigr) \\
&\quad + \bigl(|\mathrm{cov}(C_k,\mathbb S)|+|\mathbb S|\bigr)\,p\_sz.
\end{split} \label{eq:model_expand_2}
\end{align}
Hence, 
\begin{align}
\Delta L(\{\vec{\delta}\}\mid \mathbb S) &\stackrel{\eqref{eq:model_expand_1}, \eqref{eq:model_expand_2}}{=}\Bigl(\mathrm{bits}(\|\vec{\delta}\|_0)+\log_2(\|\vec{\delta}\|_1)\Bigr) \notag \\
&\quad + \Bigl( |\mathrm{cov}(C_k,\mathbb S\cup\{\vec{\delta}\})|-|\mathrm{cov}(C_k,\mathbb S)| \Bigr)p\_sz \notag \\
&\quad + \Bigl( |\mathbb S\cup\{\vec{\delta}\}|-|\mathbb S| \Bigr)p\_sz.\notag
% \label{eq:model_diff_step} 
\xRightarrow{\eqref{eq:gain}}\\
\Delta L(\{\vec{\delta}\}\mid \mathbb S) 
&=
\mathrm{bits}(\|\vec{\delta}\|_0)
+
\log_2(\|\vec{\delta}\|_1)
\notag\\
&\quad+\bigl(\gamma(\{\vec{\delta}\}\mid \mathbb S)+1\bigr)p\_sz
\label{eq:model_increment}
\end{align}

% -----------------------------
\noindent\textbf{Data part.}
As defined in \Cref{sec:problem},
\[
L(C_k \mid \mathbb S)
=
|C_k \setminus \mathrm{cov}(C_k,\mathbb S)|\cdot MC
=
|\uncov{\mathbb S}|\cdot MC.
\]
When adding $\vec{\delta}$ to $\mathbb S$, the uncovered set shrinks by exactly the number of newly covered instances:
\[
|\uncov{\mathbb S\cup\{\vec{\delta}\}})|
=
|\uncov{\mathbb S}|
-
\gamma(\{\vec{\delta}\}\mid \mathbb S).
\]
We now compute the marginal increment of the data part:
\begin{align}
\Delta L(C_k\mid (\{\vec{\delta}\}&\mid\mathbb S))
= L(C_k\mid (\mathbb S\cup\{\vec{\delta}\})) - L(C_k\mid \mathbb S) \notag\\
&=
|\uncov{\mathbb (S\cup\{\vec{\delta}\}})|\cdot MC
-
|\uncov{\mathbb S}|\cdot MC \notag\\
&=
\bigl(|\uncov{\mathbb S}|-\gamma(\{\vec{\delta}\}\mid\mathbb S)\bigr)\cdot MC
-
|\uncov{\mathbb S}|\cdot MC \notag\\
&=
-\gamma(\{\vec{\delta}\}\mid \mathbb S)\cdot MC.
\label{eq:data_increment_full}
\end{align}

% -----------------------------
\noindent\textbf{Total MDL increment.}
The total marginal increment when adding $\vec{\delta}$ to $\mathbb S$ is defined as follows:
\begin{align}
\Delta L(C_k,(\{\vec{\delta}\}\mid \mathbb S))
& =
L(C_k,(\mathbb S\cup\{\vec{\delta}\}))-L(C_k,\mathbb S) \notag
=\\
& \Delta L(\{\vec{\delta}\}\mid \mathbb S)+\Delta L(C_k\mid (\{\vec{\delta}\}\mid \mathbb S)) \notag \xRightarrow{\eqref{eq:model_increment}, \eqref{eq:data_increment_full}} \\
\Delta L(C_k, (\{\vec{\delta}\} \mid \mathbb S)) 
&= \mathrm{bits}(\|\vec{\delta}\|_0) + \log_2 (\|\vec{\delta}\|_1) \notag \\
& \quad + (\gamma(\{\vec{\delta}\} \mid \mathbb S) + 1) p\_sz - \gamma(\{\vec{\delta}\} \mid \mathbb S) \cdot MC \notag \Rightarrow \\
\Delta L(C_k, (\{\vec{\delta}\} \mid \mathbb S)) 
 &= \mathrm{bits}(\|\vec{\delta}\|_0) + \log_2 (\|\vec{\delta}\|_1)  + p\_sz \notag \\
&\quad + \gamma(\{\vec{\delta}\} \mid \mathbb S) (p\_sz - MC)
\label{eq:total_increment_closed}
\end{align}

% -----------------------------
\noindent\textbf{Supermodularity of $L(C_k,\mathbb S)$.}
Let $\A\subseteq \B\subseteq \Delta_k$ and $\vec{\delta}\in\Delta_k\setminus \B$. Then, 
\begin{equation}
\begin{split}
\Delta L(C_k, (\mathbb \{\vec{\delta}\} &\mid \A)) - \Delta L(C_k, (\{\vec{\delta}\} \mid \B))  =\\
&= \big( \Delta L(\{\vec{\delta}\} \mid \A) - \Delta L(\{\vec{\delta}\} \mid \B) \big) + \\
& + \big( \Delta L(C_k \mid (\{\vec{\delta}\} \mid \A)) - \Delta L(C_k \mid (\{\vec{\delta}\} \mid \B)) \big),
\end{split}
\label{eq:diff_split}
\end{equation}
where 
\begin{equation}
\Delta L(\{\vec{\delta}\}\mid \A)-\Delta L(\{\vec{\delta}\}\mid \B)
\stackrel{\eqref{eq:model_increment}}{=}
\bigl(\gamma(\{\vec{\delta}\}\mid \A)-\gamma(\{\vec{\delta}\}\mid \B)\bigr)\,p\_sz
\label{eq:difmodel}\end{equation}
and
\begin{equation}
\begin{split}
\Delta L(C_k\mid (\{\vec{\delta}\}\mid \A))&-\Delta L(C_k\mid (\{\vec{\delta}\}\mid \B))\\ &
\stackrel{\eqref{eq:data_increment_full}}{=}
-\bigl(\gamma(\{\vec{\delta}\}\mid \A)-\gamma(\{\vec{\delta}\}\mid \B)\bigr)\cdot MC
\end{split}
\label{eq:difdata}
\end{equation}
 
\begin{equation}
\begin{split}
\eqref{eq:diff_split} \xRightarrow{\eqref{eq:difmodel}, \eqref{eq:difdata}}\Delta L(C_k, &(\{\vec{\delta}\}\mid \A)) - \Delta L(C_k, (\{\vec{\delta}\}\mid \B)) \\
&= \bigl(\gamma(\{\vec{\delta}\}\mid \A)-\gamma(\{\vec{\delta}\}\mid \B)\bigr)(p\_sz-MC)
\end{split}
\label{eq:total_marginal_diff}
\end{equation}
It stands that  $\gamma(\{\vec{\delta}\}\mid \A)-\gamma(\{\vec{\delta}\}\mid \B)\ge 0$, due to \eqref{eq:diminishing_coverage}, and by the definition of $C$ we have $MC \ge 2p\_sz$, hence $p\_sz-MC\le 0$.
Therefore \eqref{eq:total_marginal_diff} implies
\begin{align*}
\Delta L(C_k, (\{\vec{\delta}\}\mid \A))-\Delta L(C_k, (\{\vec{\delta}\}\mid \B))\le 0 \Rightarrow \\
\Delta L(C_k, (\{\vec{\delta}\}\mid \A))\le \Delta L(C_k, (\{\vec{\delta}\}\mid \B)).
\end{align*}
In terms of the set function $L(C_k, \mathbb S)$, this is exactly
\[
L(C_k, (\A\cup\{\vec{\delta}\}))-L(C_k, \A)
\le
L(C_k, (\B\cup\{\vec{\delta}\}))-L(C_k, \B),
\]
which is the diminishing-increments condition characterizing supermodularity.
Hence $L(C_k, \mathbb S)$ is supermodular on $\Delta_k$.
\end{proof}

\begin{proposition}[Submodularity of the MDL reduction]
\label{proofs:mdl-submodular}
Given \Cref{proofs:mdl-supermodular}, the MDL reduction $\D(\mathbb S) = L(C_k,\varnothing) - L(C_k,\mathbb S)$ is a monotone submodular set function.
\end{proposition}

\begin{proof}
\noindent\textbf{Submodularity of the MDL reduction.}
Define the MDL reduction
\[
\D(\mathbb S)
=
L(C_k,\varnothing) - L(C_k,\mathbb S).
\]
Since $L(C_k, \cdot)$ is supermodular, $\D(\cdot)$ is submodular (a constant minus a supermodular function).
To see this directly, for $\A\subseteq \B$ and $\vec{\delta}\notin \B$:
\begin{align*}
\D(\A\cup\{\vec{\delta}\})-\D(\A)
&=
\bigl(L(C_k,\A) - L(C_k,(\A\cup\{\vec{\delta}\}))\bigr)\\
&\ge
\bigl(L(C_k,\B) - L(C_k,(\B\cup\{\vec{\delta}\}))\bigr)\\
&=
\D(\B\cup\{\vec{\delta}\})-\D(\B),
\end{align*}
where the inequality is exactly the supermodularity inequality proved above.
Thus $\D$ is submodular.

\noindent\textbf{Monotonicity of the MDL reduction.}
Under the admissible addition assumption (\Cref{ass:admissible}), $\gamma(\{\vec{\delta}\}\mid \mathbb S)\ge 1$.
By definition of $\vec{\delta}^*$, for any $\vec{\delta}\in\Delta_k$ we have
\begin{equation}
\begin{split}
\mathrm{bits}(\|\vec{\delta}\|_0)&+\log_2(\|\vec{\delta}\|_1)+2p\_sz \notag
\\
&\le
\mathrm{bits}(\|\vec{\delta}^*\|_0)+\log_2(\|\vec{\delta}^*\|_1)+2p\_sz \notag
= MC,
\end{split}
\end{equation}
and therefore
\begin{equation}
\mathrm{bits}(\|\vec{\delta}\|_0)+\log_2(\|\vec{\delta}\|_1)+p\_sz 
\le MC-p\_sz.
\label{eq:delta_star_bound}
\end{equation}

We also know that $\gamma(\{\vec{\delta}\}\mid \mathbb S)\ge 1$ and  
$p\_sz-MC\leq 0 $, hence
\begin{equation}
    \gamma(\{\vec{\delta}\}\mid \mathbb S)\cdot (p\_sz-MC) \leq p\_sz-MC
    \label{eq:b}
\end{equation} 

\[
\eqref{eq:total_increment_closed} \xRightarrow{\eqref{eq:delta_star_bound},\eqref{eq:b}}
\Delta L(C_k, (\{\vec{\delta}\}\mid \mathbb S))
\le
(MC-p\_sz) + (p\_sz-MC) \leq 0.
\]
Hence,
\[
\D(\mathbb S\cup\{\vec{\delta}\})-\D(\mathbb S)
=
-\Delta L(C_k, \{\vec{\delta}\}) 
\ge 0,
\]
so $\D$ is monotone over admissible additions.
\end{proof}